\documentclass{article}

\usepackage{microtype}
\usepackage{graphicx}
\usepackage{subcaption}
\usepackage{booktabs} 
\usepackage{longtable} 
\usepackage{float}
\usepackage{placeins}
\usepackage{tikz}

\usepackage{hyperref}

\usepackage[accepted]{icml2026}

\makeatletter
\renewcommand{\ICML@appearing}{Preliminary versions of this work
are to appear at the ICML 2026 Workshop on Combining Theory and
Benchmarks (CTB), the ICML 2026 Workshop on Statistical Frameworks
for Uncertainty in Agentic Systems (AgenticUQ), and the ICML 2026
Workshop on Failure Modes of Agentic AI (FAGEN), Seoul, South Korea.
Copyright 2026 by the author.}
\makeatother

\usepackage{amsmath}
\usepackage{amssymb}
\usepackage{mathtools}
\usepackage{amsthm}

\usepackage[capitalize,noabbrev]{cleveref}

\makeatletter
\def\thm@space@setup{\thm@preskip=4pt plus 1pt minus 1pt \thm@postskip=4pt plus 1pt minus 1pt}
\makeatother

\theoremstyle{plain}
\newtheorem{theorem}{Theorem}[section]
\newtheorem{proposition}[theorem]{Proposition}
\newtheorem{lemma}[theorem]{Lemma}
\newtheorem{corollary}[theorem]{Corollary}
\theoremstyle{definition}
\newtheorem{definition}[theorem]{Definition}

\theoremstyle{remark}

\crefname{corollary}{Corollary}{Corollaries}
\Crefname{corollary}{Corollary}{Corollaries}
\crefname{proposition}{Proposition}{Propositions}
\Crefname{proposition}{Proposition}{Propositions}
\crefname{lemma}{Lemma}{Lemmas}
\Crefname{lemma}{Lemma}{Lemmas}
\crefname{definition}{Definition}{Definitions}
\Crefname{definition}{Definition}{Definitions}

\newcommand{\MC}{\mathcal{M}_C}
\newcommand{\Mloc}{\mathcal{M}_a}
\newcommand{\Mprod}{\mathcal{M}^{\boxtimes}}
\newcommand{\Mjoint}{\mathcal{M}^{\star}}
\newcommand{\Pijoint}{\Pi^{\star}}

\newcommand{\Aagg}{\mathcal{A}}
\newcommand{\Cset}{\mathcal{C}}
\newcommand{\hpK}{\hat p}
\newcommand{\pF}{p_F}
\newcommand{\E}{\mathbb{E}}
\newcommand{\PP}{\mathbb{P}}
\newcommand{\F}{\mathcal{F}}
\newcommand{\eps}{\varepsilon}

\icmltitlerunning{Locally Coherent, Globally Incoherent}

\begin{document}

\twocolumn[
\icmltitle{Locally Coherent, Globally Incoherent:\\
Bounding Compositional Incoherence in Multi-Component LLM Agents}

\icmlsetsymbol{equal}{*}

\begin{icmlauthorlist}
\icmlauthor{Anany Kotawala}{princeton}
\end{icmlauthorlist}

\icmlaffiliation{princeton}{Princeton University, Princeton, NJ, USA}

\icmlcorrespondingauthor{Anany Kotawala}{akotawala@princeton.edu}

\icmlkeywords{foundation models, compositional generalization, formal performance guarantees,
              decomposable benchmark, instance-wise uncertainty, calibration,
              hierarchical projection, FTAP, anytime-valid inference,
              Brier score, multi-component systems}

\vskip 0.3in
]

\printAffiliationsAndNotice{}  

\begin{abstract}
Multi-component LLM agents assemble probabilistic claims from
components that each see only part of a joint problem; the
composition can violate basic probability axioms even when every
component is locally coherent. We formalise this \emph{locally
coherent, globally incoherent} failure via the \emph{compositional
residual} $\eps^\star$, the $L_2$ distance from the composed
quote to the joint coherent polytope, computable at runtime from
system output and the declared cross-component coupling
constraints. A product-structure dichotomy characterises when
local coherence suffices, and a Rayleigh-quotient prediction
matches the observed residual within $7\%$ on three of four
relation classes. A hierarchical Boyle--Dykstra projection
repairs the composition deterministically; an anytime-valid
e-process gives sequential coherence monitoring. Across $1{,}876$
ensemble cliques on a four-LLM mid-tier panel (frontier-panel
rerun in \S\ref{subsec:frontier}), $\eps^\star\!>\!0$ on
$33$--$94\%$ of cliques, translating to $+0.115$ nats per bet of
regret on $1{,}770$ resolved bets under the proportional
allocation rule (the gain collapses to $+0.006$ under bettors
that themselves coherentise). Three intuitive LLM-side
mitigations (retrieval, partition-aware prompting,
aggregator-LLM) each fail or regress.
\end{abstract}

\section{Introduction}
\label{sec:intro}

Foundation-model evaluation routinely reports per-question
accuracy, calibration, and proper-scoring statistics; formal,
instance-wise guarantees on \emph{system-level} performance under
composition of multiple model calls are comparatively scarce. In
multi-component agents, planners route retrieval, arithmetic, and
probability assessment to specialist subagents or tools, and each
component handles only part of a joint question. Even when every
component is calibrated and internally coherent on its assigned
questions, the aggregated belief need not be: a research component
emitting $P(\text{Republican}){=}0.6$ and a forecasting component
emitting $P(\text{Democrat}){=}0.6$ produce a $1.2$-mass quote that
no probability measure can assign, inducing a Dutch-book exposure
\citep{definetti1937foresight} that arises strictly \emph{between}
components. Per-component coherence does not, in general, repair
the composed system; we call this regime \emph{locally coherent,
globally incoherent}; \Cref{fig:iconic} shows a concrete instance
from the planner-discretion harness of \S\ref{subsec:planner_disc}.

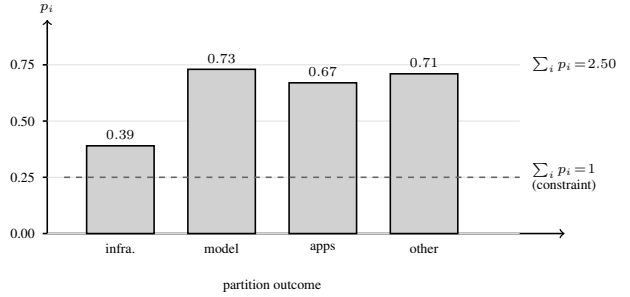
\begin{figure}[t]
\centering
\resizebox{\linewidth}{!}{\begin{tikzpicture}
\draw[->, thick] (0,0) -- (0,3.8) node[above, font=\scriptsize] {$p_i$};
\draw[->, thick] (0,0) -- (9.2,0);
\foreach \y/\lbl in {0/0.00, 1/0.25, 2/0.50, 3/0.75} {
    \draw[black!12, very thin] (0,\y) -- (8.4,\y);
    \draw[thick] (-0.08,\y) -- (0,\y);
    \node[left, font=\scriptsize] at (-0.1,\y) {\lbl};
}
\fill[black!18, draw=black, thick] (0.7,0) rectangle (1.9,1.56);
\node[font=\scriptsize] at (1.3,1.56+0.20) {$0.39$};
\fill[black!18, draw=black, thick] (2.5,0) rectangle (3.7,2.92);
\node[font=\scriptsize] at (3.1,2.92+0.20) {$0.73$};
\fill[black!18, draw=black, thick] (4.3,0) rectangle (5.5,2.68);
\node[font=\scriptsize] at (4.9,2.68+0.20) {$0.67$};
\fill[black!18, draw=black, thick] (6.1,0) rectangle (7.3,2.84);
\node[font=\scriptsize] at (6.7,2.84+0.20) {$0.71$};
\node[below, font=\scriptsize] at (1.3,-0.05) {infra.};
\node[below, font=\scriptsize] at (3.1,-0.05) {model};
\node[below, font=\scriptsize] at (4.9,-0.05) {apps};
\node[below, font=\scriptsize] at (6.7,-0.05) {other};
\node[below, font=\scriptsize] at (4.0,-0.7) {partition outcome};
\draw[black!60, dashed, thick] (0.3,1) -- (8.4,1);
\node[font=\scriptsize, anchor=west, align=left] at (8.5,1) {$\sum_i p_i\!=\!1$\\(constraint)};
\node[font=\scriptsize, anchor=west] at (8.5,3.0) {$\sum_i p_i\!=\!2.50$};
\end{tikzpicture}}
\caption{\textbf{A locally coherent, globally incoherent failure.}
Four specialists are routed one outcome each of the partition
``Largest US AI startup IPO 2026'' (infrastructure, model labs,
applications, other; from \S\ref{subsec:planner_disc}). Each
specialist sees only its assigned sector as a single Bernoulli
question, quotes a locally well-calibrated probability, and the
agent assembles the four quotes. No specialist saw that the
sectors tile the field; the assembled mass is $2.50$ and
$\eps^\star\!=\!0.749$ certifies the failure.}
\label{fig:iconic}
\end{figure}

Per-component calibration, self-consistency
\citep{wang2023selfconsistency}, and conformal prediction
\citep{angelopoulos2021gentle} act on individual model outputs and
preserve only per-output coherence properties; cross-component
logical constraints are invisible to all three. We identify the
\emph{compositional residual} $\eps^\star$, the $L_2$ distance
from the composed quote to the joint coherent polytope, as a
runtime, distribution-free certificate of system-level coherence
under owner-selected aggregation; it is computable from the
composed quote and the cross-component constraints alone. As a
constructive repair preserving specialist routing, we use a
hierarchical Boyle--Dykstra projection. We pair these with closed-form bounds, an
anytime-valid sequential test, and a decomposable benchmark whose
hardness is predicted by the theory.

A product-structure dichotomy (\Cref{thm:commute}) characterises
when local coherence suffices: under owner-selected aggregation
(each component owns a subset of joint coordinates; the aggregator
just selects what each owner already wrote), component coherence
guarantees system coherence for \emph{all} inputs if and only if
the joint polytope factorises as a Cartesian product of local
polytopes. Under any tighter coupling, locally coherent component
forecasts exist whose composition is globally incoherent, and
$\eps^\star$ certifies this failure mode (Cor.~\ref{cor:zero}).
Cor.~\ref{cor:magnitude} converts the existential dichotomy into a
closed-form magnitude prediction (Rayleigh-quotient form) computable
from the specialist panel covariance alone.

Operationally, $\eps^\star$ yields a bound
$\mathrm{Exposure}^\star\!\le\!\sqrt{m^\star}\,\eps^\star$
(Cor.~\ref{cor:syslmsr}) on system-level Dutch-book exposure under
the Fundamental Theorem of Asset Pricing (FTAP); the hierarchical
Boyle--Dykstra repair (\Cref{thm:dykstra}) drives this bound to the
numerical floor, and the corresponding e-process (a sequential test
controlling false positives at any stopping time) gives an
anytime-valid coherence test (\S\ref{sec:sequential}, \Cref{thm:audit}).
The convex-analytic machinery is classical; the contribution is
the operational reframing of $\eps^\star$ as a runtime certificate,
the dichotomy of \Cref{thm:commute}, the falsifiable magnitude
prediction of Cor.~\ref{cor:magnitude}, and the decomposable
benchmark whose hardness ordering the theory predicts.

Empirically, we evaluate $1{,}876$ ensemble cliques drawn from Paleka
\citep{paleka2024consistency} and Polymarket across four contemporary
foundation models, stratified by four logical-relation classes. The
empirical hardness ordering (partition $>$ negation $>$ disjunction
$>$ conjunction) tracks how tightly each polytope constrains the
joint quote, as predicted by \Cref{thm:commute}, and Cor.~\ref{cor:magnitude}'s
magnitude prediction matches the observed expected residual
within $7\%$ on three of the four relation classes (the
conjunction under-shoot at $0.83\!\times$ is itself predicted by
the Cor.~\ref{cor:magnitude} interior-$\bar\Pi$ regime). The
exposure bound
$\sqrt{m^\star}\eps^\star$ averages $0.137$ under naive composition
(\Cref{fig:compositional}b) and is reduced to the QP solver floor by
hierarchical projection. We report per-cell Brier with
Diebold--Mariano significance (the standard paired test for forecast
comparison), and validate with a same-model decoupling control,
leakage filtering, and a tool-augmented A/B; full protocols and a
reproducibility manifest are in the appendices.

\section{Preliminaries}
\label{sec:setup}

Throughout, starred objects denote the composed system, while
subscripts denote individual components.

\paragraph{Cliques and the coherent polytope.}
A clique $C$ is a tuple $(Q_1,\dots,Q_m, R)$ of $m$ Bernoulli
questions and logical relations $R$. By
\citet{definetti1937foresight}, the set
\(
\MC = \{r\in[0,1]^m : \exists\,\mu\in\Delta(\{0,1\}^m)\
\text{consistent with } R\}
\)
is a closed convex polytope. We denote the $L_2$ projection
$\Pi_C: [0,1]^m \to \MC$ and the residual
$\eps_C(\hpK) := \|\hpK - \Pi_C(\hpK)\|_2$.

\paragraph{$K$-sample empirical marginal.}
A forecaster supplies $K$ i.i.d.\ samples per question, yielding
$\hpK \in [0,1]^m$ with $\hpK_j = K^{-1}\sum_k Y_{j,k}$ and population
limit $\pF := \lim_{K\to\infty}\hpK$. Generically $\hpK \notin \MC$
even when $\pF\in\MC$ (due to finite-sample fluctuation).

\paragraph{Within-component JCD.}
\label{sec:percall}
Joint-Coherent Decoding is the $L_2$ projection of $K$-sample LLM
marginals onto the local coherent polytope $\MC$. The compositional
layer treats JCD as the local repair operator; it inherits a
Pythagorean Brier improvement (Theorem~\ref{thm:pyth},
App.~\ref{app:within}) and a finite-state no-Dutch-book bound
via the Fundamental Theorem of Asset Pricing (FTAP;
Cor.~\ref{cor:noarb}) bounded by
$m\cdot\eps^{\mathrm{KKT}}\!\le\!1.4\!\times\!10^{-5}$ at the
default tolerance of \textsc{OSQP}, the quadratic-program solver
we use throughout.

\paragraph{Multi-component agent.}
The agent has $k$ sub-models indexed by $a \in [k]$. Sub-model $a$
emits an empirical marginal $\hpK^{(a)} \in [0,1]^{m_a}$ on its local
question set $\mathcal{Q}_a$ with local polytope $\Mloc \subseteq
[0,1]^{m_a}$; component-level JCD produces
$\Pi_a(\hpK^{(a)}) \in \Mloc$. The agent's joint question set is
$\mathcal{Q}^\star = \bigcup_a \mathcal{Q}_a$ with size
$m^\star = |\mathcal{Q}^\star|$ and joint polytope
$\Mjoint \subseteq [0,1]^{m^\star}$, which respects (i)~each local
relation lifted to $\mathcal{Q}^\star$ and (ii)~all cross-component
\emph{coupling constraints} $\Cset$: identifications of shared
questions, cross-component logical relations like
$Q^\star = Q_a \wedge Q_b$, and partitions spanning components.

\paragraph{Aggregator.}
An aggregator
$\Aagg: \prod_a [0,1]^{m_a} \to [0,1]^{m^\star}$ assembles the
agent's quoted marginal vector. The aggregator $\Aagg$ is the
owner-selected coordinate selector of Definition~\ref{def:owner}; this
covers the specialist-routing and tool-call composition patterns
used in our experiments. We focus on owner-selection because it is the
practical pattern under compute budgets: each component owns a
subset of joint coordinates and there is no redundant elicitation
across specialists. A coord-wise averaging aggregator
$\Aagg^\mathrm{avg}(\Pi_1,\dots,\Pi_k) = k^{-1}\sum_a \Pi_a$ on the same
inputs would already lie in $\Mjoint$ by convexity, so the failure
mode of \Cref{thm:commute} is structurally absent under averaging; we
verify this empirically (\Cref{sec:exp}). The cost of averaging is
$k\times$ elicitation per coordinate, which is the regime
specialist-routing is designed to avoid.

\paragraph{The two candidate workflows.}
With $\hpK = (\hpK^{(1)},\dots,\hpK^{(k)})$ and the aggregator $\Aagg$
fixed, two natural projections into $\Mjoint$ arise:
\begin{align}
\text{Local-then-global:}\quad
& \Pijoint(\Aagg(\Pi_1\hpK^{(1)},\dots,\Pi_k\hpK^{(k)})), \label{eq:LtG}\\
\text{Direct global:}\quad
& \Pijoint(\Aagg(\hpK^{(1)},\dots,\hpK^{(k)})). \label{eq:DG}
\end{align}
The substantive question is when, and by how much, these differ.

\section{Method}
\label{sec:compose}

Throughout this section we assume the coupling set $\Cset$ is
explicitly specified; the regime of implicit $\Cset$ recovered
from agent transcripts is discussed in \S\ref{sec:conclusion}.

\subsection{When local coherence composes}
\label{subsec:dichotomy}

\begin{definition}[Compositional residual]
\label{def:compres}
For inputs $\hpK = (\hpK^{(1)},\dots,\hpK^{(k)})$ and aggregator
$\Aagg$, the \emph{compositional residual} is
\[
\begin{aligned}
\eps^\star(\hpK) \;:=\;
\bigl\|&
\Aagg(\Pi_1\hpK^{(1)},\dots,\Pi_k\hpK^{(k)})\\
&- \Pijoint\!\bigl(
\Aagg(\Pi_1\hpK^{(1)},\dots,\Pi_k\hpK^{(k)})
\bigr)\bigr\|_2.
\end{aligned}
\]
It is the $L_2$ distance from the locally-repaired composed quote
to the joint coherent polytope $\Mjoint$.
\end{definition}

\paragraph{Lifted local feasible sets.}
For each component $a$, lift $\Mloc \subseteq [0,1]^{m_a}$ to
$\Mloc^{\uparrow} \subseteq [0,1]^{m^\star}$ by leaving the
non-component coordinates free. Write
$\Mprod := \bigcap_a \Mloc^{\uparrow}$ for the joint feasibility
implied by the local constraints alone (no cross-component
coupling). $\Mjoint \subseteq \Mprod$ always; equality holds iff
the coupling set $\Cset$ is trivial.

\begin{definition}[Owner-selected aggregation]
\label{def:owner}
Fix the joint question set $\mathcal{Q}^\star$. An ownership map
$\mathrm{owner}:\mathcal{Q}^\star\!\to\![k]$ assigns each joint
coordinate to a single component. If multiple components originally
quote the same question $j$, we keep separate local coordinates
$p_{a,j}$ and $p_{b,j}$ in the product input space before owner
selection, and record the equality $p_{a,j}=p_{b,j}$ in the
coupling set $\Cset$. The aggregator $\Aagg$ selects the
coordinate owned by $\mathrm{owner}(j)$. Aggregation only
\emph{selects}; coupling \emph{enforces agreement}.
\end{definition}

\begin{theorem}[Product-structure dichotomy]
\label{thm:commute}
Under owner-selected coordinate aggregation $\Aagg$, local
coherence guarantees global coherence \emph{for all inputs}
$\hpK$ if and only if $\Mjoint = \Mprod$. When
$\Mjoint = \Mprod$, projection decomposes blockwise:
\begin{equation}
\Pijoint(\Aagg(\hpK^{(1)},\dots,\hpK^{(k)}))
= \Aagg\bigl(\Pi_1\hpK^{(1)},\dots,\Pi_k\hpK^{(k)}\bigr),
\label{eq:commute}
\end{equation}
so $\eps^\star \equiv 0$. When $\Mjoint \subsetneq \Mprod$, there
exists $r \in \Mprod \setminus \Mjoint$, and locally coherent
component forecasts realizing $r$ yield a composed quote with
$\eps^\star > 0$.
\end{theorem}

\begin{proof}[Proof (forward)]
If $\Mjoint = \Mprod$, the squared $L_2$ norm decomposes:
$\|\hpK - r\|^2 = \sum_a \|\hpK^{(a)} - r^{(a)}\|^2$ for any
$r \in \Mprod$, so the joint minimization separates into $k$
independent local minimizations. Hence
$\Pijoint(\hpK) = (\Pi_1\hpK^{(1)},\dots,\Pi_k\hpK^{(k)})$ which
equals the aggregator's output on this product space, and
$\eps^\star \equiv 0$. This is the block-product structure of
$L_2$ projection in Hilbert space
\citep[Prop.~29.3]{bauschke2017convex}.
\end{proof}

\begin{proof}[Proof (reverse, existential)]
Take any $r\in\Mprod\setminus\Mjoint$ and lift it to component
restrictions; owner-selection then reassembles $r$ from locally
coherent inputs, and $r\notin\Mjoint$ closed gives
$\eps^\star\!=\!\|r-\Pijoint(r)\|_2\!>\!0$. Full construction in
App.~\ref{app:rev_proof}.
\end{proof}

The forward direction is the block-product factorization of
$L_2$ projection in Hilbert space
\citep[Prop.~29.3]{bauschke2017convex}; the contribution is
operational, identifying when local coherence is sufficient and
when $\eps^\star$ is an appropriate online certificate. We also
refer to \Cref{thm:commute} as the \emph{non-commutation theorem},
since it characterizes exactly when local-then-global $L_2$
projection commutes with the aggregator (and when it does not).
The reverse direction is existential: it shows local coherence is
not sufficient under coupling, but does not characterize the
typical magnitude of $\eps^\star$. Section~\ref{sec:results}
estimates this empirically.

\begin{corollary}[Zero-residual characterization]
\label{cor:zero}
Even under coupling, $\eps^\star = 0$ if and only if
$\Aagg(\Pi_1\hpK^{(1)},\dots,\Pi_k\hpK^{(k)}) \in \Mjoint$. The
compositional certificate is positive on a given input exactly
when local repairs do not satisfy the cross-component constraints.
\end{corollary}

\begin{proof}
$\Pijoint$ is the identity on $\Mjoint$ by definition;
$\|x - \Pijoint(x)\|_2 = 0$ iff $x \in \Mjoint$.
\end{proof}

\subsection{Exposure interpretation}
\label{subsec:syslmsr}

\begin{corollary}[System-level FTAP exposure bound]
\label{cor:syslmsr}
Treat the coordinates of $\Aagg(\Pi_1,\dots,\Pi_k)$ as quoted
unit-pay prices on the $m^\star$ Bernoulli contracts of the joint
clique $C^\star$. If $\eps^\star>0$, the finite-state FTAP yields
a Dutch-book portfolio against the joint quote. Moreover, under the unit-stake
Logarithmic Market Scoring Rule (LMSR;
\citealp{hanson2003combinatorial}) exposure statistic used in our
experiments,
\(
\mathrm{Exposure}^\star \le \sqrt{m^\star}\,\eps^\star.
\)
Every individual sub-model's component-level Dutch-book exposure
is zero by construction up to the QP tolerance (since
$\Pi_a(\hpK^{(a)}) \in \Mloc$); any positive joint exposure
is therefore attributable to cross-component incoherence, modulo
the numerical solver floor.
\end{corollary}

\begin{proof}
If the quoted vector lies outside $\Mjoint$, the separating
hyperplane form of de Finetti's finite-state FTAP gives a
portfolio with nonnegative payoff in every coherent state and
strictly positive price discrepancy. For the relation classes used
in our experiments, the unit-stake exposure statistic is the
positive violation of a defining halfspace (or a sum of mutually
exclusive positive violations). If
$h(x)=(\langle a,x\rangle-b)_+$ is any such violation and
$x^\Pi=\Pijoint(x)$, then
$h(x)\le |\langle a,x-x^\Pi\rangle|\le \|a\|_2\|x-x^\Pi\|_2$.
For negation, conjunction, disjunction, and partition constraints
$\|a\|_2\le \sqrt{m^\star}$ after the relation is embedded in the
$m^\star$ joint coordinate system, yielding
$\mathrm{Exposure}^\star \le \sqrt{m^\star}\eps^\star$. Each
$\Pi_a(\hpK^{(a)}) \in \Mloc$ has zero local residual to $\Mloc$
by definition, so the per-component exposure diagnostic is zero up
to numerical solver tolerance.
\end{proof}

\begin{corollary}[System-level Pythagorean Brier improvement]
\label{cor:sysbrier}
Let $p^\star \in \Mjoint$ be the true joint marginal vector and
let $r$ be any composed quote. Then
\(
\|\Pijoint(r) - p^\star\|_2^2 \;\le\; \|r - p^\star\|_2^2
- \|r - \Pijoint(r)\|_2^2.
\)
In particular, taking
$r = \Aagg(\Pi_1\hpK^{(1)},\dots,\Pi_k\hpK^{(k)})$ yields a
deterministic, sample-path Brier-improvement guarantee for
hierarchical JCD whose slack
$\|r-\Pijoint(r)\|_2^2 = (\eps^\star)^2$ is largest exactly where
the compositional residual is largest.
\end{corollary}

\begin{proof}
The Hilbert-projection argument of \Cref{thm:pyth} applied to
$\Pijoint$ on $\Mjoint$.
\end{proof}

\begin{corollary}[Brier lift under $\{0,1\}$ labels: condition and reversal]
\label{cor:label_lift}
Let $y\!\in\!\{0,1\}^{m^\star}$ be a resolved label vector and let
$p^\star := \E[y]$ denote the true marginal.
For any composed quote $r$,
\(
\E_y[\mathrm{Brier}(r,y) - \mathrm{Brier}(\Pijoint(r),y)] \;=\;
\|r-p^\star\|_2^2 - \|\Pijoint(r)-p^\star\|_2^2.
\)
If $p^\star \in \Mjoint$, this difference is at least
$\|r-\Pijoint(r)\|_2^2 = (\eps^\star)^2 \ge 0$ by Cor.~\ref{cor:sysbrier};
hierarchical JCD improves expected Brier (in expectation over $y$;
the sample-path geometric bound is Cor.~\ref{cor:sysbrier}).
If $p^\star \notin \Mjoint$ (label incoherence: a fraction of resolutions
violate the assumed logical structure), the difference can be negative;
projection can amplify the error against $p^\star$.
\end{corollary}

\noindent The condition $p^\star \in \Mjoint$ is the precise assumption
under which Cor.~\ref{cor:sysbrier} transfers from the geometric to the
predictive setting. Empirically (\S\ref{sec:results}), the lift holds with
$p\!<\!10^{-15}$ on negation, conjunction, and partition; on
disjunction the gain is borderline ($\Delta\mathrm{Brier}\!=\!{+}0.0027$, $p\!=\!0.07$),
diagnosing a small fraction of disjunction cliques whose resolved
labels deviate from the strict $A\!\vee\!B$ relation: this is
exactly the $p^\star\!\notin\!\Mjoint$ regime in which
Cor.~\ref{cor:label_lift} predicts the lift can reverse.

\subsection{Disagreement controls the residual}
\label{subsec:disagree}

\begin{proposition}[Disagreement upper bound on $\eps^\star$]
\label{prop:disagree}
Let $r \in \Mjoint$ be any \emph{globally coherent reference
forecast}, that is, any joint quote that already lies in the
joint polytope. Under owner-selected coordinate aggregation,
\[
\eps^\star(\hpK)
\;\le\;
\bigl\|\Aagg(\Pi_1\hpK^{(1)},\dots,\Pi_k\hpK^{(k)}) - r\bigr\|_2.
\]
A natural choice for $r$ is the JCD-projected forecast
$\Pi_\beta(\hpK^{(\beta)})$ of any component $\beta$ that itself
lies in $\Mjoint$, for instance a component routed the entire
joint clique. The right-hand side is then the $L_2$ disagreement
between the composed quote and the reference. It vanishes when
all coordinates are assigned to the same component, so
inter-component disagreement on disputed coordinates is a
measurable proxy that both explains and upper-bounds the
residual.
\end{proposition}

\begin{proof}
By the variational characterization of $L_2$ projection onto a
closed convex set, $\eps^\star = \|x-\Pijoint(x)\|_2 \le \|x-r\|_2$
for any $r \in \Mjoint$, where
$x = \Aagg(\Pi_1\hpK^{(1)},\dots,\Pi_k\hpK^{(k)})$.
\end{proof}

Let $\Delta_j$ denote the per-coordinate disagreement between the
assigned LLM and a globally-coherent reference $\beta$. For each
relation class the bound takes a closed form keyed to the
constraint normal: $|\Delta_1{+}\Delta_2|/\sqrt{2}$ on negation,
$|\sum_j\Delta_j|/\sqrt{m^\star}$ on partition, and $\|\Delta\|_2$
on the Fr\'echet conjunction/disjunction polytope (full
derivations in App.~\ref{app:size3}).

Empirically the bound holds on $100\%$ of the $1{,}876$ cliques.
Mean bound-to-residual ratios are $1.30$ (negation), $2.34$
(disjunction), $3.00$ (conjunction), and $4.39$ (partition); the
bound is tightest where the constraint normal aligns with the
disagreement direction.

\subsection{Quantitative magnitude prediction}
\label{subsec:magnitude}

Under uniform random owner-selection, the expected residual admits a
closed form keyed to the empirical specialist covariance restricted to
the constraint normal that defines the joint polytope. Let
$\bar\Pi := k^{-1}\sum_a \Pi_a \in \Mjoint$ (by convexity) and
$D := \mathrm{diag}(\Sigma_\Pi)$ where
$\Sigma_\Pi := k^{-1}\sum_a (\Pi_a - \bar\Pi)(\Pi_a - \bar\Pi)^\top$.
Independent assignment across coordinates makes $D$ the effective
covariance: $\E_\sigma[(x-\bar\Pi)(x-\bar\Pi)^\top] = D$.

\begin{corollary}[Magnitude of $\eps^\star$, Rayleigh-quotient form]
\label{cor:magnitude}
Suppose the joint polytope is locally defined by a single binding
halfspace with normal $a_R$ at the projection. Under uniform i.i.d.\
owner-selection,
\begin{equation}
\E_\sigma\!\bigl[(\eps^\star)^2\bigr] \;=\;
\kappa_R \cdot \frac{a_R^\top D\, a_R}{\|a_R\|_2^2},
\label{eq:rayleigh}
\end{equation}
where $\kappa_R = 1$ when the binding constraint is an equality
($a_R^\top \bar\Pi = b_R$). When the binding constraint is an
inequality with $\bar\Pi$ on the boundary, the law of $a_R^\top(x-\bar\Pi)$
is the empirical distribution of $\{a_R^\top(\Pi_a-\bar\Pi)\}_{a\in[k]}$;
if this scalar is symmetric around $0$ then $\kappa_R = \tfrac12$
exactly, and otherwise $\kappa_R \in [0,1]$ with $\kappa_R \approx \tfrac12$
when the panel is approximately balanced about the boundary. The
upper bound $\E_\sigma[(\eps^\star)^2] \le a_R^\top D a_R/\|a_R\|_2^2$
holds in all inequality cases, and the generic
$\E_\sigma[(\eps^\star)^2] \le \mathrm{tr}(D)$ holds without further
assumptions. For the relation classes of \Cref{sec:exp}:
$a_{\textsc{neg}} = (1,1)$ (equality, $\kappa = 1$);
$a_{\textsc{partition}} = \mathbf 1_{m^\star}$ (equality, $\kappa = 1$);
$a_{\textsc{and}}, a_{\textsc{or}}$ are Fr\'echet halfspace normals
(inequality, $\kappa \approx \tfrac12$ under the symmetric-panel
approximation; conjunction's empirical $\kappa\!\approx\!0.83\cdot\tfrac12$
reflects $\bar\Pi$ off the boundary, see~\S\ref{sec:results}).
\end{corollary}

\noindent\emph{Proof sketch.} Equality: $\eps^\star =
|a_R^\top(x-\bar\Pi)|/\|a_R\|$, square and use
$\E_\sigma[a_R^\top(x-\bar\Pi)(x-\bar\Pi)^\top a_R]=a_R^\top D a_R$.
Inequality with $\bar\Pi$ on the boundary: $\eps^\star =
(a_R^\top(x-\bar\Pi))_+/\|a_R\|$; for a zero-mean random variable $X$,
$\E[X_+^2] = \tfrac12 \E[X^2]$ when $X$ is symmetric, otherwise
$\E[X_+^2]/\E[X^2] \in [0,1]$ depending on the upper-tail mass.
Full proof and the generic-bound case in App.~\ref{app:magnitude}.

Combined with Cor.~\ref{cor:label_lift}, this yields a closed-form
\emph{Brier tax}
$\E[\mathrm{Brier}(r,y) - \mathrm{Brier}(\Pijoint(r),y)]
\ge \kappa_R\, a_R^\top D a_R/\|a_R\|^2$ when
$p^\star\!\in\!\Mjoint$, computable from the panel covariance
alone. The three couplings recurring in our experiments are
\emph{shared-question disagreement} (two specialists quote the same
$Q$ at different prices), \emph{negation across components}
($p_a\!+\!p_b\!=\!1$), and \emph{partition across components}
($\sum_a p_a\!=\!1$).

\subsection{Hierarchical repair}
\label{sec:hier}

The repair is constructive. Let
$\Pi_\Cset: [0,1]^{m^\star}\to \Cset$ denote the $L_2$ projection
onto the coupling set, an intersection of finitely many
half-spaces.

\begin{theorem}[Hierarchical projection convergence]
\label{thm:dykstra}
Let $\{\Pi_1,\dots,\Pi_k,\Pi_\Cset\}$ be the family of $L_2$
projections onto the local polytopes $\{\Mloc\}$ (lifted to
$[0,1]^{m^\star}$) and the coupling set $\Cset$. Their intersection
is the nonempty joint polytope $\Mjoint$. The Boyle-Dykstra
cyclic-projection iteration
\citep{boyle1986dykstra} initialized at any $r_0 \in [0,1]^{m^\star}$
generates a sequence converging in $L_2$ to $\Pijoint(r_0)$.
\end{theorem}

\begin{proof}[Sketch]
$\Mjoint = \bigl(\bigcap_a \Mloc^{\uparrow}\bigr) \cap \Cset$ where
$\Mloc^{\uparrow}$ is $\Mloc$ lifted to $[0,1]^{m^\star}$ (free on
non-component coordinates). Each $\Pi_\bullet$ is the
$L_2$ projection onto a closed convex set, and each $\Mloc^{\uparrow}$
and $\Cset$ is closed convex. By
\citet[Thm.~9.31]{bauschke2017convex} (Boyle-Dykstra), cyclic
$L_2$ projection over a finite family of closed convex sets
converges to the projection onto their intersection.
\end{proof}

\paragraph{Per-iteration cost.}
Each $\Pi_a$ and $\Pi_\Cset$ is a small QP. The Paleka (neg/and/or)
and Polymarket (partition) relation classes admit closed-form
projections (App.~\ref{app:size3}); we use direct $\Pijoint$ via
\textsc{OSQP} \citep{stellato2020osqp} for uniformity across
relation classes, since the partition simplex projection and the
negation closed form are mathematically equivalent to OSQP's
converged solution at the reported tolerance.

\paragraph{Where the cyclic machinery actually bites.}
On the equality-coupled relations (negation, partition) the joint
projection collapses to a single closed-form step
(App.~\ref{app:size3}): partition repair is simplex projection
\citep{wang2013simplex}, which in the no-clipping majority is
literally subtracting $(\sum_i p_i\!-\!1)/m^\star$ from each
coordinate. The Boyle--Dykstra cycle is mathematically equivalent to
this but does real iterative work only on the Fr\'echet polytopes
(conjunction, disjunction), where multiple half-spaces can be
simultaneously active. The headline residual is therefore largest
on the relation (partition) where the repair is simplest, and
smallest where the cyclic machinery is needed; the contribution is
operational, not algorithmic.

\subsection{Runtime use}
\label{sec:runtime}

The residual supports three deployment modes that share the same
computation and differ only in how the agent acts on it. In
\emph{monitor} mode the agent logs $\eps^\star$ and surfaces large
residuals to an operator. In \emph{repair} mode it replaces the naive composed quote with
$\Pijoint(\cdot)$ before passing it downstream. By
\Cref{thm:dykstra} and Cor.~\ref{cor:syslmsr}, the repaired quote
lies in $\Mjoint$ up to solver tolerance and the exposure
diagnostic is driven to numerical zero. In
\emph{abstain-or-escalate} mode it refuses to act when
$\eps^\star\!>\!\tau$, a risk--coverage knob whose calibration to
budgeted exposure is outside our scope.

\subsection{Sequential monitoring: anytime-valid coherence test}
\label{sec:sequential}

In long-horizon agent deployments the agent generates a
compositional residual stream $(\eps^\star_t)_{t\ge 1}$ across
successive composition steps. The spatial guarantee of
\Cref{sec:compose} certifies incoherence at each step; the
following temporal guarantee lets an operator certify
\emph{persistent} incoherence across the stream without fixing
a stopping horizon in advance.

Under the null $H_0$ that the composed population quote is
jointly coherent ($p_F^\star\!\in\!\Mjoint$ at each step),
the e-process
\(
E_t(\lambda) := \prod_{s=1}^{t}
\exp\!\bigl(\lambda(\eps_s^{\star 2} - m_s^\star/(4K_s))
- \lambda^2 m_s^\star/(2K_s)\bigr)
\)
is a non-negative $\mathcal{F}_t$-supermartingale
(Lemma~\ref{lem:hoeff}, App.~\ref{app:audit}). By Ville's
inequality \citep{ville1939etude,howard2021time}, the stopping
rule $\tau_\alpha := \inf\{t : E_t(\lambda)\ge 1/\alpha\}$
controls Type-I error at level~$\alpha$ uniformly over all
stopping times (\Cref{thm:audit}). An operator monitors $E_t$
online and escalates the agent pipeline when it crosses
$1/\alpha$; no fixed-horizon commitment is required. The full
derivation and empirical demonstration on four LLM residual
streams appear in \Cref{app:audit}.

\section{Experimental setup}
\label{sec:exp}
\label{subsec:exp_setup}

We evaluate on Paleka \citep{paleka2024consistency} (negation,
conjunction, disjunction, paraphrase; $134$ cliques per checker)
and Polymarket ($67$ partition cliques after leakage control,
expanded to $268$ ensemble instances by four seeds). Both
benchmarks were chosen because they target the regime the theory
predicts will be tightest: Paleka was constructed by
\citet{paleka2024consistency} for cross-question consistency
checking, and Polymarket's partition events are the unit-mass
constraint regime ($\sum_i p_i = 1$) where the polytope coupling
is most restrictive.

\paragraph{Models.}
Anthropic \texttt{claude-haiku-4-5-20251001} (via the Anthropic
Messages API), OpenAI \texttt{gpt-5.4-mini} and \texttt{gpt-5.4-nano}
(via Azure OpenAI), and Meta \texttt{llama-3.3-70b-versatile} (via
Groq). Frontier-panel evaluations additionally use OpenAI
\texttt{gpt-5.5} (via Azure OpenAI). These are the public API model
identifiers as returned by each provider on 30 April 2026; raw
responses, prompts, and per-clique residuals are released with the
supplementary material. Each specialist produces
$K{=}8$
verbalized probability samples per question at temperature $0.7$.

\paragraph{Leakage control.}
Polymarket events were filtered so that
every market in every retained event resolved strictly after our
model snapshot date of 30 April 2026; the inclusion cutoff was
1 May 2026, and all retained events resolved on or after that
date. Paleka cliques inherit the resolution-status verification
of \citet{paleka2024consistency}. Full filtering protocol in
Appendix~\ref{app:cliques}.

\paragraph{Scope of the evaluation.}
The bulk of our evaluation is a controlled routing simulation:
clique coordinates are assigned i.i.d.\ to specialists from the
panel, and the composed quote is reassembled by owner-selected
aggregation. The planner-discretion harness
(\S\ref{subsec:planner_disc}) and routing-protocol ladder
(App.~\ref{app:context}) are the closest stand-ins for a deployed
multi-component agent; both are small ($n\!=\!20$ and $n\!=\!100$
partitions respectively). Claims about end-to-end deployed agents
should be read with that scope in mind.

\paragraph{Compositional ensemble.}
For each clique with $m$ questions, we draw $4$ random
seeds; each seed assigns each of the $m$ questions to one of the
$4$ LLMs uniformly i.i.d.; the agent's quoted marginal on coordinate
$j$ is the JCD-projected marginal of the assigned LLM on that
coordinate. The source forecast from each LLM is JCD-coherent on
the full clique before coordinate selection; the ensemble residual
$\eps^\star$ therefore reflects \emph{exclusively}
cross-component disagreement.

\paragraph{Benchmark properties.}
The four relation classes admit closed-form local projections
(App.~\ref{app:size3}) and span a range of polytope geometries
from a single equality (negation), to a high-dimensional simplex
(partition), to the Fr\'echet box (and/or). The coupling set
$\Cset$ controls the gap between $\Mjoint$ and $\Mprod$, and random
vs.\ structured routing brackets cross-model mixing.
\Cref{thm:commute}'s prediction $\eps^\star\!\equiv\!0$ when
$\Mjoint\!=\!\Mprod$ is verified at the QP solver floor by the
no-composition reference (\Cref{fig:compositional}b, gray dashed),
a falsifiable test of the dichotomy.

\begin{figure*}[!t]
\centering
\includegraphics[width=\linewidth]{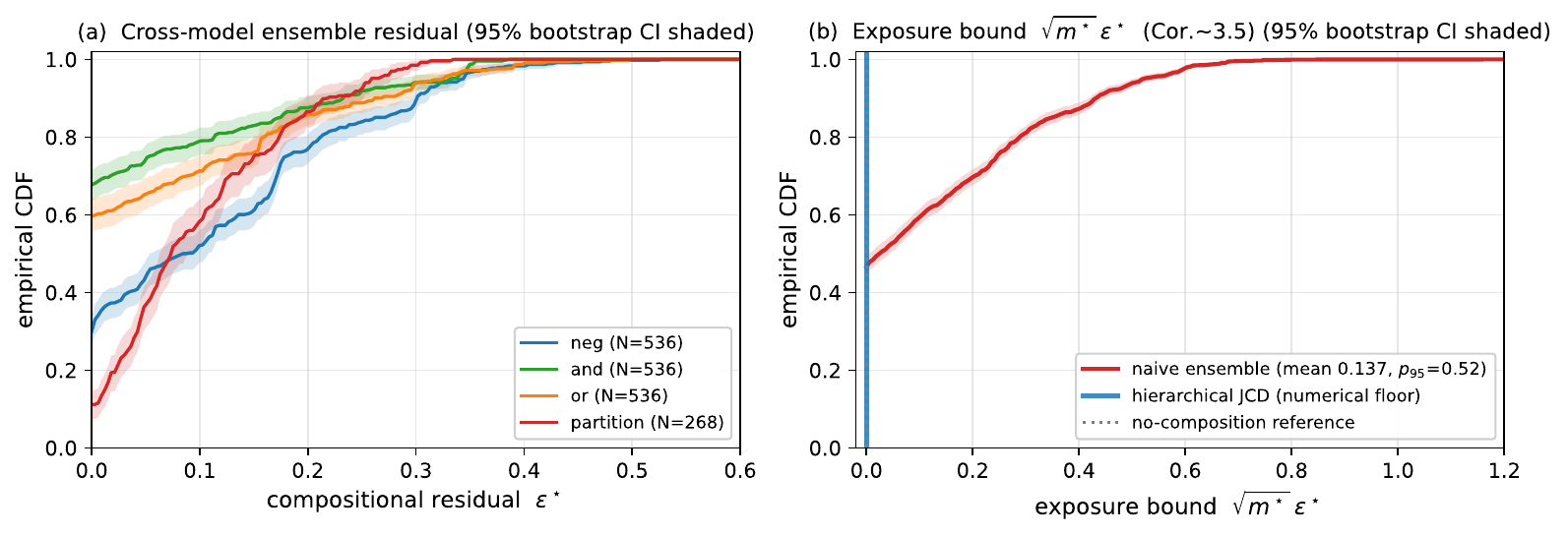}
\caption{\textbf{Compositional residuals and the exposure bound.}
\emph{(a)} ECDF of $\eps^\star$ across $N{=}1{,}876$ ensemble
cliques, stratified by relation. Each component is per-call
JCD-projected, so any positive $\eps^\star$ is cross-component
incoherence introduced by composition.
\emph{(b)} ECDF of the exposure bound
$\sqrt{m^\star}\eps^\star$ (Cor.~\ref{cor:syslmsr}) under three
regimes: naive composition (mean $0.137$), hierarchical JCD (at the
OSQP floor), and a single-LLM no-composition reference. Shaded
bands are $95\%$ bootstrap CIs ($B{=}1000$).}
\label{fig:compositional}
\end{figure*}

\section{Results}
\label{sec:results}

The random-assignment experiment ($N\!=\!134$ cliques
$\times\,4$ seeds $=\!536$ per relation for Paleka, plus
$N\!=\!67\!\times\!4\!=\!268$ for Polymarket; total $1{,}876$)
and structured-routing variant ($N\!=\!469$, one deterministic
seed) bracket cross-model mixing.

\subsection{Compositional residual and controls}
\label{subsec:compres_results}

\paragraph{Compositional residual under random routing.}
The compositional residual is strictly positive on $94\%$ (partition),
$66\%$ (negation), $43\%$ (disjunction), $33\%$ (conjunction) of
cliques (\Cref{fig:compositional}a). Mean $\eps^\star$ ranges $0.058$
(and) to $0.118$ (partition) on the full $N\!=\!1{,}876$ ensemble.
The projection-step ablation
(App.~\ref{app:projection_ablation}, Table~\ref{tab:proj_ablation})
reports $0.054/0.098$ on the same relations because it is restricted
to the resolved-label subset on which paired Brier is computed
($\sim\!85\%$ of the ensemble); the $\sim\!20\%$ magnitude shift on
partition reflects that unresolved partition cliques tend to be
larger-$m^\star$ events with proportionally higher residuals.
\emph{None} is visible per-component,
where each $\Pi_a\!\in\!\Mloc$ has zero residual. Hierarchical JCD
reduces the mean exposure bound $\sqrt{m^\star}\eps^\star$ to
$1.7\!\times\!10^{-14}$ on the $1{,}876$-clique ensemble
(per-clique post-repair $\eps^\star$ reaches the
\textsc{OSQP} primal-residual floor of
$\le\!1.5\!\times\!10^{-16}$ on partition and negation, where
closed-form projections terminate at machine precision;
the ensemble mean is dominated by the and/or relations, which
bottom out at the iterative-OSQP tolerance instead), matching the no-composition reference (a single LLM on every
coordinate; \Cref{fig:compositional}b, gray) while retaining
specialist routing.

\paragraph{Empirical validation of the magnitude prediction.}
Cor.~\ref{cor:magnitude}'s Rayleigh-quotient prediction
$\kappa_R\,a_R^\top D a_R/\|a_R\|^2$ matches the observed
expected residual $\E_\sigma[(\eps^\star)^2]$ across all four
relation classes (Table~\ref{tab:theory_obs}; per-clique Pearson
$0.69$--$0.89$). The conjunction under-shoot reflects $\bar\Pi$
typically interior to the binding Fr\'echet halfspace; on the
remaining three relations the prediction is tight to within
$7\%$. Both $\eps^\star$ and the $\sqrt{m^\star}\eps^\star$
exposure bound are therefore predictable from the panel covariance
alone, before any routing.

\begin{table}[h]
\centering
\small
\begin{tabular}{lccc}
\toprule
Relation & $\kappa_R$ & Obs./Pred. & Match \\
\midrule
\textsc{Negation}    & $1$            & $1.054$ & $94.6\%$ \\
\textsc{Partition}   & $1$            & $1.069$ & $93.1\%$ \\
\textsc{Disjunction} & $\tfrac{1}{2}$ & $1.026$ & $97.4\%$ \\
\textsc{Conjunction} & $\tfrac{1}{2}$ & $0.830$ & $83.0\%$ \\
\bottomrule
\end{tabular}
\caption{Theory vs.\ observed expected squared residual.
\emph{Obs./Pred.}\ is the ratio of measured
$\E_\sigma[(\eps^\star)^2]$ to Cor.~\ref{cor:magnitude}'s
Rayleigh-quotient prediction across the $1{,}876$ ensemble cliques;
\emph{Match} is $1-|1-\text{Obs./Pred.}|$. The conjunction
under-shoot is the $\bar\Pi$-interior regime predicted by
Cor.~\ref{cor:magnitude}.}
\label{tab:theory_obs}
\end{table}

\paragraph{Same-model decoupling control.}
To separate coordinate isolation from cross-model capability
divergence, we re-ran with four fresh-seed runs of one model
(Claude-Haiku-4.5, $N{=}80$ cliques per relation, $40$
partitions). Mean $\eps^\star$ falls to $0.025/0.019/0.040/0.058$
on neg/and/or/partition: $22$--$60\%$ of the residual persists at
same-model, with cross-model heterogeneity amplifying by
$1.7$--$4.5\times$. A greedy-decoding ($T{=}0$) control retains
$\eps^\star = 0.025$--$0.096$ across relations
(App.~\ref{app:t0_control}); on negation and disjunction $T{=}0$
amplifies the residual relative to $T{=}0.7$, ruling out sampling
noise as the source.

\paragraph{Compositional Brier against resolved labels.}
We adopt the convention $\Delta\mathrm{Brier} = \mathrm{Brier}_{\mathrm{JCD}} - \mathrm{Brier}_{\mathrm{naive}}$
throughout, so a negative $\Delta$ indicates JCD improvement.
Paired over the $1{,}876$ cliques, hierarchical JCD reduces
compositional Brier on \textsc{Neg}
($\Delta\mathrm{Brier}{=}-0.0137$, $p{<}10^{-43}$), \textsc{And}
($\Delta\mathrm{Brier}{=}-0.0076$, $p{<}10^{-16}$), and \textsc{Partition}
($\Delta\mathrm{Brier}{=}-0.0048$, $p{<}10^{-23}$). \textsc{Or} is
marginal ($\Delta\mathrm{Brier}{=}{+}0.0027$, $p{=}0.07$): this is
the predicted reversal regime of Cor.~\ref{cor:label_lift}, where
label noise on a small fraction of disjunction resolutions pushes
$p^\star$ outside $\Mjoint$, and Cor.~\ref{cor:label_lift}'s
condition-and-reversal allows projection to amplify error there
rather than reduce it.

\paragraph{Disagreement-driven mechanism.}
Per-clique $\eps^\star$ regressed on the assignment-aware reference
disagreement $\min_\beta\|x{-}\Pi_\beta\|_2$
(Prop.~\ref{prop:disagree}) gives $R^2{=}0.93$ on negation (slope
$0.83$) and $R^2{=}0.58$--$0.64$ on and/or/partition (slope
$0.59$--$0.63$). The constraint normal aligns most closely with
the disagreement direction on negation; on higher-dimensional
polytopes the projection contracts components of disagreement
orthogonal to the binding halfspace, loosening the bound.

\paragraph{Robustness to routing topology and retrieval grounding.}
Under structured routing (one model per relation family;
App.~\ref{app:structured}), $\eps^\star\!>\!0$ on $98.5\%$ of
negation, $36\%$ partition, $16\%$ disjunction, $12\%$ conjunction
cliques; mean $\sqrt{m^\star}\eps^\star\!=\!0.117$ vs.\ $0.137$
under random routing. Retrieval grounding (top-$5$ DuckDuckGo
snippets per specialist) on $30$ partitions shifts mean $\eps^\star$
from $0.260$ to $0.283$; $20/30$ partitions are equal-or-worse with
retrieval (App.~\ref{app:toolaug}).

\paragraph{Routing-protocol ladder.}
On $100$ multi-candidate partitions (App.~\ref{app:context}) we
walk the routing protocol from fixed naive owner-selection to
autonomous tool-using planners. Prompt-engineering under
owner-selection drops mean $\eps^\star$ from $0.235$
(\textsc{isolated}) to $0.090$ (\textsc{listed}) to $0.081$
(\textsc{full}); \textsc{isolated}\,$\to$\,\textsc{listed} is
significant ($p\!<\!10^{-11}$), \textsc{listed}\,$\to$\,\textsc{full}
is not ($p\!=\!0.18$), and $99\%$ of partitions retain
$\eps^\star\!>\!10^{-3}$ under \textsc{full}. Tool-using planners
with delegate-and-submit tools further reduce $\eps^\star$: a
non-reasoning planner reaches $\eps^\star\!=\!0.009/0.004$ at
$\sim\!4.8\!\times m$ calls; a frontier reasoning planner reaches
$0/200$ positive residuals at $\sim\!6.7\!\times m$ calls. The
hierarchical projection reaches $\eps^\star\!\to\!0$ at
$1\!\times m$ specialist calls (\Cref{thm:dykstra}).

\subsection{Intuitive LLM-side mitigations fail}
\label{subsec:mitigations}

The retrieval-grounding and routing-ladder results above each
isolate one obvious LLM-side fix. We collect them with a third
candidate, delegating coherentisation to an aggregator LLM, and
compare all three to the geometric repair. The comparison is
deliberately stacked in favour of the geometric repair: the
projection is handed the explicit coupling set $\Cset$ as a
typed input, while the three prompt-based methods have to infer
it from natural-language context. The genuinely hard regime
where $\Cset$ is implicit in free-form planner-tool transcripts
is outside the scope of all four interventions here and is
discussed as future work (\S\ref{sec:conclusion}).
\Cref{tab:mitigations} reports mean $\eps^\star$, the fraction of
partitions retaining $\eps^\star\!>\!0.05$, and the fraction
regressing relative to naive composition. Retrieval grounding
\emph{regresses} on $20/30$ partitions: mean $\eps^\star$ rises
from $0.260$ to $0.283$ on the matched $30$-partition baseline.
(Table~\ref{tab:mitigations}'s ``Naive'' row reports $0.214$, the
$100$-partition mean used for the LLM-aggregator comparison; the
two subsets are flagged in the caption.) Partition-aware prompting
helps on
average ($0.066$) but $53\%$ of partitions still exceed the $0.05$
band and $5/30$ regress. LLM-as-aggregator reduces mean
$\eps^\star$ to $0.028$ but $15\%$ of partitions still exceed
$0.05$ and $7/100$ regress against the naive baseline. The
geometric repair eliminates the residual to the QP solver floor
($\le 1.5\!\times\!10^{-16}$) at $1$ms per partition with no
extra LLM call. Prompt-only and aggregator-LLM mitigations can
\emph{harm} already-coherent quotes: across both prompt-side
interventions, $9/10$ worsening cases had baseline
$\eps^\star\!<\!0.05$, so $\eps^\star\!>\!\tau$ also serves as a
gate for when these prompt-based repairs are safe to apply.

\begin{table}[h]
\centering
\small
\begin{tabular}{@{}lrrrl@{}}
\toprule
Method & $\overline{\eps^\star}$ & frac.\ $>\!0.05$ & \% reg. & cost \\
\midrule
Naive             & $0.214$            & $0.86$ & ---   & 0 \\
Retrieval         & $0.283$            & $0.87$ & $67$  & 1 search \\
Partition prompt  & $0.066$            & $0.53$ & $17$  & 0 \\
LLM-aggregator    & $0.028$            & $0.15$ & $7$   & 1 LLM \\
\midrule
Hier.\ JCD        & $\le\!10^{-16}$    & $0$    & $0$   & 1 QP \\
\bottomrule
\end{tabular}

\caption{\textbf{Three intuitive LLM-side mitigations vs.\ the
geometric repair.} Rows use the evaluation subset for the
corresponding mitigation: retrieval-augmented and partition-aware
prompt on $30$ matched partitions; naive, LLM-aggregator, and
hierarchical JCD on $100$ partitions. The geometric repair
deterministically eliminates $\eps^\star$ without extra LLM
calls.}
\label{tab:mitigations}
\end{table}

\subsection{Planner-discretion harness}
\label{subsec:planner_disc}

To check that the failure is not an artefact of random assignment,
we re-run the experiment with a planner LLM (Claude-Haiku-4.5)
that, for each of $20$ live-style partitions, picks a specialist
per outcome from the four-LLM roster and emits its own
per-specialist context block; specialists then sample $K{=}8$
verbalized probabilities under the planner's prompt and we compute
$\eps^\star$ as in \S\ref{subsec:exp_setup}. Mean $\eps^\star =
0.113$ versus $0.231$ under random assignment on the same
partitions ($15/20$ improve), but $\eps^\star\!>\!0$ on $20/20$.
The planner's routing histogram (GPT-5.4-mini $50\%$, Claude-Haiku
$31\%$, Llama-70b $13\%$, GPT-nano $6\%$) shows no self-routing.

\subsection{Mechanism probe: coupling-visible elicitation}
\label{subsec:coupling}

The mechanism claim of \S\ref{sec:compose} is that $\eps^\star$
arises because specialists do not see the coupling constraint.
We test this directly. On the same $20$ partitions of
\S\ref{subsec:planner_disc}, we re-elicit each specialist's quote
under two paired conditions. \textsc{Blind} keeps the
planner-chosen routing of \S\ref{subsec:planner_disc} but strips
the planner-emitted context, so the specialist sees only its
assigned outcome as a single Bernoulli question. \textsc{Informed}
keeps the same routing and additionally shows the specialist the
partition label, all sibling outcomes, the explicit
$\sum_i p_i\!=\!1$ constraint, and the peers' \textsc{Blind} quotes
on the other outcomes. Each (partition, condition) cell is
replicated across four independent $K\!=\!8$ sampling rounds,
paired \textsc{Blind}$\to$\textsc{Informed} within each round.
\textsc{Informed} varies two things at once (coupling visibility
and peer-quote disclosure), so the comparison identifies
coupling-visible elicitation with peer context as the intervention,
not constraint visibility in isolation; full setup in
App.~\ref{app:coupling}.

\begin{table}[h]
\centering
\small
\begin{tabular}{@{}lrr@{}}
\toprule
                                 & \textsc{Blind}  & \textsc{Informed} \\
\midrule
Mean $\eps^\star$ (20 partitions)& $0.519$         & $0.298$           \\
Improved ($\Delta>0$)            & \multicolumn{2}{c}{$16/20$}         \\
Unchanged                        & \multicolumn{2}{c}{$1/20$}          \\
Worsened ($\Delta<0$)            & \multicolumn{2}{c}{$3/20$}          \\
\midrule
Paired $\Delta\eps^\star$ (CI)   & \multicolumn{2}{c}{$+0.221$ $[+0.173,+0.270]$} \\
Wilcoxon $W$                     & \multicolumn{2}{c}{$257$, $p\!=\!2.6\!\times\!10^{-10}$} \\
\bottomrule
\end{tabular}

\caption{\textbf{Coupling-visibility intervention on the
20 highest-$\eps^\star$ partitions} of the planner harness
(\S\ref{subsec:planner_disc}). $\Delta = \eps^\star_{\textsc{Blind}} - \eps^\star_{\textsc{Informed}}$; the
paired-bootstrap CI excludes zero and the Wilcoxon signed-rank
test rejects at $p\!=\!2.6\!\times\!10^{-10}$.}
\label{tab:coupling}
\end{table}

Disclosing the coupling constraint (and peer quotes) reduces
$\eps^\star$ as the theory predicts: paired mean $\eps^\star$
falls from $0.519$ \textsc{Blind} to $0.298$ \textsc{Informed}
on $16/20$ partitions, with $1$ unchanged and $3$ worsened.
The \textsc{Blind} mean of $0.519$ exceeds both the
random-assignment mean ($0.231$) and the planner-harness mean
($0.113$) on these same $20$ partitions. Two protocol differences
combine: the planner harness lets a planner LLM see the partition
when choosing routing (and emit its own per-specialist context),
while \textsc{Blind} strips all such context; and the four
\textsc{Blind} seeds are independent $K\!=\!8$ sampling rounds
on the same planner-chosen routing, so cross-seed variance does
not average out as it does over random reshuffles. We did not
isolate the contributions further; the residual gap is observed,
not theory-derived.
The three worsening partitions are non-payoff multi-way forecasts
where \textsc{Blind} specialists were already roughly normalised;
per-partition detail in App.~\ref{app:coupling}. The intervention is partial, not a
fix: $\eps^\star$ remains positive on every \textsc{Informed}
partition, consistent with the LLM-side mitigations result
(\S\ref{subsec:mitigations}) that no purely prompt-based
intervention eliminates the residual. The geometric repair
remains the only intervention that drives $\eps^\star$ to the
QP floor on every partition.

\subsection{Frontier-panel rerun}
\label{subsec:frontier}

To test whether capability scaling closes the gap, we rerun the
random-assignment ensemble on the $67$ Polymarket partition cliques
with a frontier-only specialist roster: Claude-Opus-4.7, GPT-5.5
(reasoning disabled for apples-to-apples), DeepSeek-V3.2, and
Llama-4-Maverick-17B-128E (FP8). The protocol matches
\S\ref{subsec:exp_setup} exactly ($K{=}8$, simplex pre-projection,
$4$ seeds, $268$ ensemble bets). The residual persists:
$\eps^\star\!>\!0$ on $97.8\%$ of the $268$ clique-seed bets,
slightly above the mid-tier panel's $94\%$ on the same Polymarket
slice. Mean $\eps^\star$ drops from $0.118$ to $0.072$ ($-39\%$):
frontier models reduce residual magnitude on partitions but do not
eliminate the failure mode. Brier improvement under hierarchical
JCD remains highly significant ($\Delta\mathrm{Brier} = -0.012$,
$95\%$ paired-bootstrap CI $[-0.015, -0.009]$); the log-payoff
gain trends positive ($+0.107$ nats per bet) but is noisier on
this narrower slice. Table~\ref{tab:frontier} reports the full
frontier-vs-mid-tier comparison.

\begin{table}[h]
\centering
\resizebox{\linewidth}{!}{\begin{tabular}{@{}lcc@{}}
\toprule
Metric                       & Frontier                    & Mid-tier                   \\
\midrule
$\Pr(\eps^\star\!>\!0)$      & $0.978$                     & $0.940$                    \\
$\Pr(\eps^\star\!>\!0.05)$   & $0.433$                     & ---                        \\
$\Pr(\eps^\star\!>\!0.10)$   & $0.257$                     & ---                        \\
mean $\eps^\star$            & $0.072$                     & $0.118$                    \\
median $\eps^\star$          & $0.042$                     & ---                        \\
$\Delta\mathrm{Brier}$ (JCD) & $-0.012$ $[-0.015,-0.009]$  & $-0.018$ $[-0.021,-0.015]$ \\
$\Delta\log p$ (JCD)         & $+0.107$ $[-0.12,+0.36]$    & $+0.60$ $[+0.15,+1.15]$    \\
$\Delta\log p$ (oracle)      & $-0.051$ $[-0.42,+0.32]$    & $\approx 0$                \\
\bottomrule
\end{tabular}
}
\caption{\textbf{Frontier vs.\ mid-tier panel} ($n\!=\!268$
partition bets each for $\eps^\star$ and $\Delta\mathrm{Brier}$;
$n\!=\!162$ unique-YES subset for $\Delta\log p$). $95\%$
paired-bootstrap CIs over (clique, seed) pairs; the mid-tier
numbers are the partition-only projection of
\S\ref{subsec:compres_results} and \S\ref{subsec:regret}.}
\label{tab:frontier}
\end{table}

\subsection{Downstream-decision regret}
\label{subsec:regret}

For each of $1{,}770$ resolved ensemble bets from \S\ref{sec:exp}
(we drop $106/1{,}876$ where the assigned specialist lacks a
forecast for the assigned coordinate, a pipeline artefact on
partition events) we convert the composed quote $\hat p^\star$ into
a proportional allocation
$w_i = \max(\hat p^\star_i,0)/\sum_j\max(\hat p^\star_j,0)$ and
measure (i)~Brier loss against the binary resolution vector and
(ii)~the realised log-payoff $\log w_{\text{winner}}$ on the
$1{,}252$ bets with a unique YES. The proportional rule passes
the agent's incoherent quote through to the bettor's allocation,
which exposes the behavioural cost of compositional incoherence
in its rawest form; stricter rules (truncated Kelly, max-entropy)
that themselves coherentise the quote partially absorb this cost
into the bettor's strategy, so their log-payoff gains under
JCD are correspondingly smaller (Table~\ref{tab:alloc}).

Hierarchical JCD lowers mean Brier ($\Delta\mathrm{Brier}=-0.018$,
paired-bootstrap $95\%$ CI over (clique, seed): $[-0.021,-0.015]$)
and improves mean log-payoff by $+0.115$ nats per bet
($[+0.056,+0.187]$). JCD matches the single-LLM oracle's gain
($\Delta\mathrm{Brier}= -0.017$, CI $[-0.035,-0.000]$) with a
confidence interval $\sim\!6\times$ narrower. The effect
concentrates on partitions (naive loses $0.60$ nats/bet vs.\ JCD,
CI $[+0.15, +1.15]$, $n{=}162$). The Pythagorean Brier improvement
of Cor.~\ref{cor:sysbrier} is thus realised on resolved labels,
not only on the synthetic $\eps^{\star 2}$ slack.

\paragraph{Regret stratifies by certificate.}
The realised regret rises monotonically with $\eps^\star$
(Table~\ref{tab:quartile}): the bottom two quartiles (where
$\eps^\star\!\le\!0.004$) carry no measurable Brier or log-payoff
gain, while the top quartile (where $\eps^\star\!>\!0.154$)
accounts for $0.054$ Brier and $0.221$ nats of regret per bet.
This justifies the runtime-gating recipe of
\S\ref{subsec:gate}: applying the projection only when
$\eps^\star\!>\!\tau$ targets precisely the high-regret bets.

\begin{table}[h]
\centering
\small
\begin{tabular}{@{}lcccc@{}}
\toprule
$\eps^\star$ quartile & range            & $\overline{\Delta\mathrm{Brier}}$ & $\overline{\Delta\log p}$ & $n$       \\
\midrule
Q1                    & $[0.000, 0.000]$ & $0.000$ & $0.000$ & $\sim 440$ \\
Q2                    & $(0.000, 0.004]$ & $0.000$ & $0.000$ & $\sim 440$ \\
Q3                    & $(0.004, 0.154]$ & $0.012$ & $0.123$ & $\sim 440$ \\
Q4                    & $(0.154, 0.545]$ & $0.054$ & $0.221$ & $\sim 440$ \\
\bottomrule
\end{tabular}

\caption{\textbf{Mean realised regret stratifies by
$\eps^\star$ quartile.} $n_{Q_1}\!=\!0$ for the bottom-quartile
log-payoff column because all such bets have JCD$\equiv$naive.}
\label{tab:quartile}
\end{table}

\paragraph{Allocation-rule sensitivity.}
The log-payoff gain depends on how a downstream bettor allocates
on top of the agent's quote. Brier is allocation-rule independent
($-0.018$ across all rules). Under the proportional rule (the
agent's incoherence is passed through as allocation slack),
JCD beats naive by $+0.168$ nats/bet on the $1{,}252$ unique-YES
subset ($95\%$ CI $[+0.083, +0.263]$); the $+0.115$ figure cited
in the abstract is the same comparison averaged over the full
$1{,}770$-bet ensemble (non-unique-YES bets contribute zero
log-payoff). Under stricter rules that themselves coherentise
the quote (truncated Kelly, max-entropy), the gain shrinks to
$+0.006$ nats: the bettor has already absorbed the agent's
incoherence into its own strategy (Table~\ref{tab:alloc}).

\begin{table}[h]
\centering
\small
\begin{tabular}{@{}lcc@{}}
\toprule
Allocation rule    & $\Delta_{\text{naive}\to\text{JCD}}$ & $95\%$ CI            \\
\midrule
proportional       & $+0.168$                             & $[+0.083, +0.263]$   \\
truncated Kelly    & $+0.006$                             & $[-0.003, +0.017]$   \\
max-entropy        & $+0.006$                             & $[-0.003, +0.017]$   \\
\bottomrule
\end{tabular}

\caption{\textbf{JCD vs.\ naive log-payoff under three
allocation rules} ($n\!=\!1{,}252$ unique-YES bets; paired
bootstrap). Proportional is the clearest behavioural-harm
measure; Kelly and max-entropy themselves coherentise.}
\label{tab:alloc}
\end{table}

\subsection{Runtime gating thresholds}
\label{subsec:gate}

The deployment recommendation is to apply repair (or escalate)
only when $\eps^\star\!>\!\tau$. We calibrate $\tau$ on the
1,770-bet ensemble of \S\ref{subsec:regret}, treating top-quartile
log-payoff regret as the harm signal. Brier-regret is
mechanically tied to $\eps^{\star 2}$ by Cor.~\ref{cor:sysbrier};
log-payoff regret depends on which coordinate resolved YES and
is therefore not geometrically tied to $\eps^\star$, making it
the genuinely informative discrimination test.

\paragraph{Discrimination.}
$\eps^\star$ predicts top-quartile log-payoff regret with $5$-fold
cross-validated AUC $0.775\!\pm\!0.044$; the geometric Brier-regret
check gives CV-AUC $0.974\!\pm\!0.012$. Mean realised log-payoff
regret stratifies cleanly by $\eps^\star$ quartile
($0.000, 0.000, 0.123, 0.221$ nats, bottom to top;
Table~\ref{tab:quartile}).

\paragraph{Operating thresholds.}
A high-recall threshold $\tau\!=\!0.15$ catches $91\%$ of harmful
bets at a $25\%$ alert rate and $3\%$ false-alarm rate; a
high-precision threshold $\tau\!=\!0.22$ catches $51\%$ at
$14\%$ alert rate and $1.5\%$ FPR (Table~\ref{tab:gating}). Both
thresholds are stable under $5$-fold CV (at $\tau\!=\!0.15$:
mean alert rate $25\%$, mean recall $90\%$). For deployment we
recommend $\tau\!\approx\!0.15$ when repair is cheap and
$\tau\!\approx\!0.22$ when escalation is costly.

\begin{table}[h]
\centering
\resizebox{\linewidth}{!}{\begin{tabular}{@{}lcccc@{}}
\toprule
Mode & $\tau$ & alert rate & harm capture & FPR \\
\midrule
high-recall ($\ge 90\%$ capture)    & $0.153$ & $25.1\%$ & $91.0\%$ ($90.4\%$) & $3.1\%$ \\
balanced                            & $0.180$ & $19.5\%$ & $73.5\%$            & $2.3\%$ \\
high-precision ($\ge 50\%$ capture) & $0.222$ & $13.8\%$ & $50.7\%$            & $1.5\%$ \\
\bottomrule
\end{tabular}
}
\caption{\textbf{Operating points for the $\eps^\star$ gate},
computed in-sample then verified by $5$-fold CV (mean shown
in parentheses where it differs from the in-sample number).}
\label{tab:gating}
\end{table}

\subsection{Per-component layer (summary)}
\label{subsec:percomp}

The local JCD layer behaves as \Cref{thm:pyth} predicts. Sorting
Paleka cliques by pre-projection $\eps$, JCD's Brier gain stratifies
to the most-incoherent quartile ($0.040$ at $m{=}2$, $0.013$ at $m{=}3$;
$0.12$ on Polymarket at $m{\ge}4$); other quartiles show no detectable
gain (\Cref{fig:headline_polymarket}). Stacked with per-question
Platt calibration \citep{platt1999probabilistic} (baseline B6 of
App.~\ref{app:fulltables}), JCD reduces mean Brier on every
$(\text{model},\text{checker})$ cell, by $0.037$ on Paleka and $0.046$
on Polymarket; Diebold--Mariano $p<10^{-4}$ on every model. A Murphy
decomposition (\Cref{tab:murphy_real}) attributes the gain to
calibration ($36\%$ reliability reduction) rather than
over-extremization. Per-relation B6 gains on Paleka ($-21\%$ and,
$-37\%$ or, $-4\%$ neg, $-1\%$ paraphrase) track local-polytope
restrictiveness.

\paragraph{Reproducibility.}
Code, prompts, sample dumps, per-clique JCD residuals, the
ensemble-assignment seed, and the full hierarchical-projection
pipeline are documented in \Cref{app:cliques} and released at
\url{https://github.com/akotawala10/composition-incoherence-icml};
$K$-sweep and the sequential test appear in
Appendix~\ref{app:audit}.

\section{Conclusion}
\label{sec:conclusion}

The compositional residual $\eps^\star$ supplies the instance-wise,
distribution-free coherence guarantee that per-component calibration,
self-consistency, and conformal prediction do not: it is computable
at deployment from system output and the cross-component constraints
alone, certifies a Dutch-book exposure no per-component repair can
remove, and admits a deterministic projection-based fix. The
product-structure dichotomy (\Cref{thm:commute}) characterises when
local coherence implies global coherence; Cor.~\ref{cor:magnitude}
predicts the typical residual from the panel covariance, matching
observed $\E_\sigma[(\eps^\star)^2]$ within $7\%$ on three of
four relation classes; the hierarchical Boyle--Dykstra projection
eliminates $\eps^\star$ at $1\!\times\!m$ elicitation cost; and on
$1{,}770$ resolved bets the certificate yields $+0.115$ nats per bet
over naive composition. The guarantees are orthogonal to
per-output conformal prediction \citep{angelopoulos2021gentle} and
complementary to opinion pooling
\citep{genest1986combining,abbas2009log}, which requires every
expert to forecast every coordinate.

\paragraph{Limitations.}
The Rayleigh-quotient prediction of Cor.~\ref{cor:magnitude} is
exact for negation and partition cliques but conservative when the
panel mean $\bar\Pi$ lies in the interior of the binding Fr\'echet
halfspace, which produces the $0.83\!\times$ under-shoot we observe
on conjunction.

The principal scope condition is that the coupling set $\Cset$ is
specified explicitly: every theorem in \S\ref{sec:compose} treats
$\Cset$ as a finite, externally given list of half-spaces and
equalities over joint coordinates (shared-question identifications,
cross-component logical relations, partition-mass constraints).
This covers structured tool-use and function-calling deployments,
where the planner emits declared output schemas with typed
sub-question fields and $\Cset$ is recoverable from the schema
without language understanding. Our planner-discretion harness
(\S\ref{sec:exp}) and routing-protocol ladder
(App.~\ref{app:context}) target this regime.

It does not cover deployments in which $\Cset$ is implicit in
free-form chain-of-thought or in unstructured planner--tool
transcripts. There, $\eps^\star$ remains well-defined and the
Boyle--Dykstra repair still applies on any subsequently specified
polytope, but the certificate cannot be computed without first
recovering $\Cset$ from the trace. Partial routes to that
recovery exist as components: an NLI-style classifier over
sub-question pairs, schema-enforcing planner protocols, and
symbolic trace logging at the tool-call boundary. A unified
inference step remains open; we list it as the first future-work
item below.

Finally, our resolved-label experiment relies on Polymarket
resolutions, which are noisy on a small fraction of disjunction
cliques; this pushes $p^\star$ outside $\Mjoint$ on those cliques
and is the regime where Cor.~\ref{cor:label_lift} does not
guarantee improvement.

\paragraph{Future work.}
First, the coupling set $\Cset$
is currently provided externally; recovering $\Cset$ from
unstructured planner--tool traces would extend $\eps^\star$ to
free-form agent transcripts, where the cross-component constraints
are implicit in the prompt structure rather than typed. Second, the
$L_2$/Brier pairing is one instance of a broader matched-Bregman
recipe; substituting log-loss or spherical scoring rules yields
analogous projections under their respective Bregman geometries
\citep{banerjee2005clustering}, with the residual generalising to a
Bregman gap. Third, the spatial certificate of \Cref{thm:commute}
and the temporal e-process of \S\ref{sec:sequential} compose
step-by-step but ignore long-horizon dependence between steps;
a sequential certificate that exploits inter-step structure is
open.

\section*{Broader Impact}

A runtime certificate for system-level coherence can be wired into
deployed multi-component pipelines as a guardrail, surfacing
incoherent compositions before they reach a downstream decision
policy. This is a reliability and auditability gain for high-stakes
settings such as forecasting, decision support, and regulated
probabilistic reasoning. The corresponding risk is over-trust: a
small $\eps^\star$ certifies coherence of the quoted belief but
says nothing about whether the underlying specialists are
calibrated against the true distribution, and a projected quote may
be confidently wrong on the very inputs where its components
disagree most. Operators deploying the repair should therefore
report $\eps^\star$ alongside conventional calibration and coverage
diagnostics, not as a substitute for them.

\bibliography{refs}

@inproceedings{paleka2024consistency,
  title     = {Consistency Checks for Language Model Forecasters},
  author    = {Paleka, Daniel and Sudhir, Abhimanyu Pallavi and Alvarez,
               Alejandro and Bhat, Vineeth and Shen, Adam and Wang, Evan
               and Tram\`er, Florian},
  booktitle = {International Conference on Learning Representations (ICLR)},
  year      = {2025},
  note      = {arXiv:2412.18544}
}

@article{stellato2020osqp,
  title   = {{OSQP}: an Operator Splitting Solver for Quadratic Programs},
  author  = {Stellato, Bartolomeo and Banjac, Goran and Goulart, Paul and
             Bemporad, Alberto and Boyd, Stephen},
  journal = {Mathematical Programming Computation},
  volume  = {12},
  number  = {4},
  pages   = {637--672},
  year    = {2020},
  doi     = {10.1007/s12532-020-00179-2}
}

@article{hanson2003combinatorial,
  title   = {Combinatorial Information Market Design},
  author  = {Hanson, Robin},
  journal = {Information Systems Frontiers},
  volume  = {5},
  number  = {1},
  pages   = {107--119},
  year    = {2003},
  doi     = {10.1023/A:1022058209073}
}

@inproceedings{wang2013simplex,
  title     = {Projection onto the probability simplex: An efficient algorithm
               with a simple proof, and an application},
  author    = {Wang, Weiran and Carreira-Perpi{\~n}{\'a}n, Miguel {\'A}.},
  booktitle = {arXiv preprint arXiv:1309.1541},
  year      = {2013}
}

@article{boyle1986dykstra,
  title   = {A method for finding projections onto the intersection of convex
             sets in {H}ilbert spaces},
  author  = {Boyle, James P. and Dykstra, Richard L.},
  journal = {Lecture Notes in Statistics},
  volume  = {37},
  pages   = {28--47},
  year    = {1986},
  doi     = {10.1007/978-1-4613-9940-7_3}
}

@incollection{platt1999probabilistic,
  title     = {Probabilistic Outputs for Support Vector Machines and
               Comparisons to Regularized Likelihood Methods},
  author    = {Platt, John C.},
  booktitle = {Advances in Large Margin Classifiers},
  pages     = {61--74},
  year      = {1999},
  publisher = {MIT Press}
}

@article{definetti1937foresight,
  title   = {La pr{\'e}vision: ses lois logiques, ses sources subjectives},
  author  = {de Finetti, Bruno},
  journal = {Annales de l'Institut Henri Poincar{\'e}},
  volume  = {7},
  pages   = {1--68},
  year    = {1937},
  note    = {English translation in \emph{Studies in Subjective Probability},
             Kyburg \& Smokler (eds.), Wiley 1964.}
}

@inproceedings{wang2023selfconsistency,
  title     = {Self-Consistency Improves Chain of Thought Reasoning in
               Language Models},
  author    = {Wang, Xuezhi and Wei, Jason and Schuurmans, Dale and Le,
               Quoc and Chi, Ed and Narang, Sharan and Chowdhery, Aakanksha
               and Zhou, Denny},
  booktitle = {International Conference on Learning Representations (ICLR)},
  year      = {2023}
}

@book{bauschke2017convex,
  title     = {Convex Analysis and Monotone Operator Theory in Hilbert Spaces},
  author    = {Bauschke, Heinz H. and Combettes, Patrick L.},
  edition   = {2nd},
  publisher = {Springer},
  year      = {2017},
  doi       = {10.1007/978-3-319-48311-5}
}

@book{ville1939etude,
  title     = {{\'E}tude critique de la notion de collectif},
  author    = {Ville, Jean},
  publisher = {Gauthier-Villars},
  address   = {Paris},
  year      = {1939}
}

@article{howard2021time,
  title   = {Time-uniform, nonparametric, nonasymptotic confidence sequences},
  author  = {Howard, Steven R. and Ramdas, Aaditya and McAuliffe, Jon and
             Sekhon, Jasjeet},
  journal = {The Annals of Statistics},
  volume  = {49},
  number  = {2},
  pages   = {1055--1080},
  year    = {2021},
  doi     = {10.1214/20-AOS1991}
}

@article{ramdas2023game,
  title   = {Game-theoretic statistics and safe anytime-valid inference},
  author  = {Ramdas, Aaditya and Gr{\"u}nwald, Peter and Vovk, Vladimir
             and Shafer, Glenn},
  journal = {Statistical Science},
  volume  = {38},
  number  = {4},
  pages   = {576--601},
  year    = {2023},
  doi     = {10.1214/23-STS894}
}

@article{angelopoulos2021gentle,
  title   = {A Gentle Introduction to Conformal Prediction and
             Distribution-Free Uncertainty Quantification},
  author  = {Angelopoulos, Anastasios N. and Bates, Stephen},
  journal = {arXiv preprint arXiv:2107.07511},
  year    = {2021}
}

@incollection{mcdiarmid1989method,
  title     = {On the Method of Bounded Differences},
  author    = {McDiarmid, Colin},
  booktitle = {Surveys in Combinatorics, 1989},
  editor    = {Siemons, J.},
  series    = {London Mathematical Society Lecture Note Series},
  volume    = {141},
  pages     = {148--188},
  publisher = {Cambridge University Press},
  year      = {1989},
  doi       = {10.1017/CBO9781107359949.008}
}

@article{banerjee2005clustering,
  author  = {Banerjee, Arindam and Merugu, Srujana and Dhillon, Inderjit S. and Ghosh, Joydeep},
  title   = {Clustering with {B}regman divergences},
  journal = {Journal of Machine Learning Research},
  volume  = {6},
  pages   = {1705--1749},
  year    = {2005}
}

@article{genest1986combining,
  author  = {Genest, Christian and Zidek, James V.},
  title   = {Combining probability distributions: A critique and an annotated bibliography},
  journal = {Statistical Science},
  volume  = {1},
  number  = {1},
  pages   = {114--135},
  year    = {1986}
}

@article{abbas2009log,
  author  = {Abbas, Ali E.},
  title   = {A {K}ullback--{L}eibler view of linear and log-linear pools},
  journal = {Decision Analysis},
  volume  = {6},
  number  = {1},
  pages   = {25--37},
  year    = {2009}
}
\bibliographystyle{icml2026}

\newpage
\appendix
\onecolumn
\raggedbottom

\section{Within-component projection facts}
\label{app:within}

\begin{theorem}[Pythagorean Brier improvement]
\label{thm:pyth}
Let $p^*\!\in\!\MC$ be the ground-truth marginal for clique $C$ and
$\Pi(\hpK)$ the $L_2$ projection of $\hpK$ onto $\MC$. Then for any
$\hpK\!\in\!\mathbb{R}^m$,
$\|\Pi(\hpK) - p^*\|_2^2 \le \|\hpK - p^*\|_2^2 - \|\hpK - \Pi(\hpK)\|_2^2$.
The slack $\|\hpK - \Pi(\hpK)\|_2^2$ is the squared distance from
coherence; Brier improvement is largest where pre-projection
incoherence is largest.
\end{theorem}

\begin{proof}
Hilbert projection theorem
\citep[Thm.~3.16]{bauschke2017convex} gives
$\langle \hpK - \Pi(\hpK),\,p^* - \Pi(\hpK)\rangle \le 0$. Expand
$\|\hpK - p^*\|^2$ around $\Pi(\hpK)$ and drop the non-negative
cross term.
\end{proof}

\begin{corollary}[Per-component finite-state FTAP]
\label{cor:noarb}
Treating $\Pi(\hpK)\!\in\!\MC$ as quoted unit-pay prices on the $m$
Bernoulli contracts of $C$, the finite-state Fundamental Theorem of
Asset Pricing rules out a state-independent Dutch book against the
ideal projection. Under a numerical QP solution the unit-stake
exposure diagnostic is bounded by
$m\cdot\eps^{\mathrm{KKT}}\!\le\!1.4\!\times\!10^{-5}$ at
\textsc{OSQP} default tolerance.
\end{corollary}

\section{Proof of Theorem~\ref{thm:commute} (reverse direction)}
\label{app:rev_proof}

Suppose $\Mjoint \subsetneq \Mprod$. Pick any
$r \in \Mprod \setminus \Mjoint$. For each component $a$, set
$\hpK^{(a)}$ to the restriction of $r$ to component $a$'s
coordinates: $\hpK^{(a)}_i := r_{j(a,i)}$ where $j(a,i)$ is the
joint index that component~$a$'s $i$-th local coordinate maps to
under Definition~\ref{def:owner}; under shared-question coupling,
set $\hpK^{(a)}_{i_a} = \hpK^{(b)}_{i_b} = r_j$ whenever components
$a$ and $b$ both own the duplicate of joint coordinate $j$. Then
each restriction lies in $\Mloc$ since $r \in \Mprod$, so
$\Pi_a(\hpK^{(a)})$ equals that restriction and the owner-selected
aggregator reassembles $\Aagg(\Pi_1,\dots,\Pi_k) = r$. Since
$r \notin \Mjoint$ and $\Mjoint$ is closed, $\Pijoint(r) \neq r$
and $\eps^\star = \|r - \Pijoint(r)\|_2 > 0$.

\section{Closed-form local projections}
\label{app:size3}

\paragraph{Negation ($m{=}2$, $r_1 + r_2 = 1$).}
$\Pi_{\MC}(\hpK) = \bigl(\tfrac{1}{2}(1 + \hat p_1 - \hat p_2),\,
\tfrac{1}{2}(1 - \hat p_1 + \hat p_2)\bigr)$;
deterministic Dutch-book on raw $\hpK$ is $|\hat p_1 + \hat p_2 - 1|$.

\paragraph{Conjunction ($m{=}3$, $r_3{=}r_1\!\wedge\!r_2$).}
$\MC = \{r\in[0,1]^3 : \max(0, r_1{+}r_2{-}1) \le r_3 \le \min(r_1, r_2)\}$
is the Fr\'echet polytope; $\Pi_{\MC}$ is the unique fixed point of
joint clipping on $r_3$ and feasibility-preserving adjustments on
$(r_1, r_2)$, computed by Dykstra's cyclic projection algorithm
\citep{boyle1986dykstra} on the six halfspaces (four box, two
Fr\'echet) and converging to machine precision in $<\!200$
iterations. Deterministic Dutch-book on raw $\hpK$ is
$\max(0, \hat p_3 - \min(\hat p_1, \hat p_2)) +
\max(0, \max(0,\hat p_1{+}\hat p_2{-}1) - \hat p_3)$.

\paragraph{Disjunction ($m{=}3$, $r_3{=}r_1\!\vee\!r_2$).}
$\MC = \{r\!\in\![0,1]^3 : \max(r_1, r_2) \le r_3 \le \min(1, r_1{+}r_2)\}$
via $P(A{\vee}B) = P(A){+}P(B){-}P(A{\wedge}B)$; same Dykstra
recipe.

\paragraph{Partition ($\sum_i r_i = 1$, $r_i\ge 0$).}
$\Pi_{\MC}$ is the simplex projection of \citet{wang2013simplex}
in $O(n\log n)$. Deterministic Dutch-book on raw $\hpK$ is
$|\sum_i \hat p_i - 1|$.

\section{Sequential test on the residual stream}
\label{app:audit}

In long-horizon agent deployments, the compositional residual
stream $(\eps^\star_t)_{t \ge 1}$ can be monitored over time;
the per-component test below extends to the system level once a
bounded-difference envelope for the chosen aggregation rule is
specified. Under coordinate-selection aggregation with disjoint
owners, $\eps_t^{\star 2}$ inherits the envelope below with
$m_t$ replaced by $m^\star_t$ and $K_t$ replaced by
$\min_a K_{a,t}$. The compositional theory of \Cref{sec:compose} is
spatial: it controls coherence at a single time step across $k$
components. The residual stream $(\eps_t)_{t\ge 1}$ that an agent
generates over time admits a complementary \emph{temporal}
guarantee.

\paragraph{Setup.}
A clique stream $(C_t)_{t\ge 1}$ with sizes $m_t$ and $K_t$ samples
per question yields residuals $\eps_t = \|\hpK_t - \Pi(\hpK_t)\|_2$.
Assume the $K_t m_t$ samples at time $t$ are conditionally
independent given $\F_{t-1}$, each bounded in $[0,1]$ (Bernoulli
parses or verbalized-probability emissions both satisfy this; the
bounded-difference constant below uses the worst-case $[0,1]$
variance $1/4$, which majorises the verbalized-sample variance),
with conditional mean vector $\pF^{(t)}$. Define the null
\(
H_0:\,\pF^{(t)} \in \MC^{(t)}\ \text{for all $t$}.
\)

\begin{lemma}[Bounded-difference MGF bound]
\label{lem:hoeff}
Under $H_0$, $\eps_t^2 \in [0, m_t]$ with conditional mean
$\E[\eps_t^2 \mid \F_{t-1}] \le m_t/(4K_t)$, and for every
$\lambda > 0$,
\(
\E\!\left[\exp\!\bigl(\lambda(\eps_t^2 - m_t/(4K_t))\bigr) \mid
\F_{t-1}\right]
\le \exp\!\bigl(\lambda^2 m_t / (2K_t)\bigr).
\)
\end{lemma}

\begin{proof}
$\eps_t \le \|\hpK_t - \pF^{(t)}\|_2$ under $H_0$ by projection
optimality on $\MC$. The squared norm
$Z_t = \|\hpK_t - \pF^{(t)}\|_2^2$ is a function of $K_t m_t$
i.i.d.\ $[0,1]$-bounded samples with bounded differences $2/K_t$ per
sample; McDiarmid's bounded-difference inequality
\citep{mcdiarmid1989method} yields the sub-Gaussian MGF bound
$\E[\exp(\lambda(Z_t-\E Z_t))]\le \exp(\lambda^2 m_t/(2K_t))$, and
$\E Z_t \le m_t/(4K_t)$ from the per-component variance bound.
\end{proof}

\begin{theorem}[Anytime-valid coherence test]
\label{thm:audit}
For any $\lambda > 0$,
\(
E_t(\lambda) := \prod_{s=1}^{t} \exp(\lambda(\eps_s^2 - m_s/(4K_s))
- \lambda^2 m_s/(2K_s)),\ E_0=1,
\)
is a non-negative $\F_t$-supermartingale under $H_0$. Hence by
Ville's inequality \citep{ville1939etude,howard2021time}, for every
$\alpha \in (0,1)$ and every stopping time $\tau$,
\(
\PP\bigl(\sup_{t\le\tau} E_t(\lambda) \ge 1/\alpha \mid H_0\bigr)
\le \alpha.
\)
The decision rule
$\tau_\alpha := \inf\{t : E_t \ge 1/\alpha\}$ controls Type-I error
at level $\alpha$ uniformly over stopping rules.
\end{theorem}

\begin{proof}
Each factor of $E_t(\lambda)$ is non-negative. By
Lemma~\ref{lem:hoeff},
$\E[E_t/E_{t-1}\mid\F_{t-1}] \le \exp(\lambda^2 m_t/(2K_t))
\exp(-\lambda^2 m_t/(2K_t)) = 1$, the supermartingale property.
Ville's inequality yields the maximal bound, uniform over stopping
times by optional stopping.
\end{proof}

\paragraph{Tuning $\lambda$.}
For the alternative
$\E[\eps_t^2] \to m_t/(4K_t) + \delta$, the
bounded-difference-optimal bet is
$\lambda^* = \delta K_t / m_t$; we mix over a log grid
$\lambda \in \{0.05,0.1,0.2,0.5,1,2,5,10\}$ in our experiments
(\Cref{fig:eprocess}), retaining anytime-valid coverage by
linearity \citep{ramdas2023game}.

\paragraph{Empirical demonstration.}
\Cref{fig:eprocess} runs the test on each model's per-clique
residual stream ($N{=}603$ cliques per model, fixed-seed shuffle).
Three of four models reject $H_0$ at $\alpha\le 10^{-4}$ within
$\le\!422$ cliques (Llama-3.3-70b at $t=16$, GPT-5.4-nano at
$t=177$, GPT-5.4-mini at $t=422$); Claude-Haiku-4.5 has mean
residual below the $H_0$-envelope at $K{=}8$ and is not rejected
within $603$ cliques, demonstrating that the test calibrates
power against the bounded-difference envelope. The same
construction applies to the compositional residual stream
$(\eps^\star_t)$ from \Cref{sec:compose} after replacing the
per-component envelope by the corresponding envelope for the
agent's aggregation rule.

\begin{figure}[h]
\centering
\includegraphics[width=\linewidth]{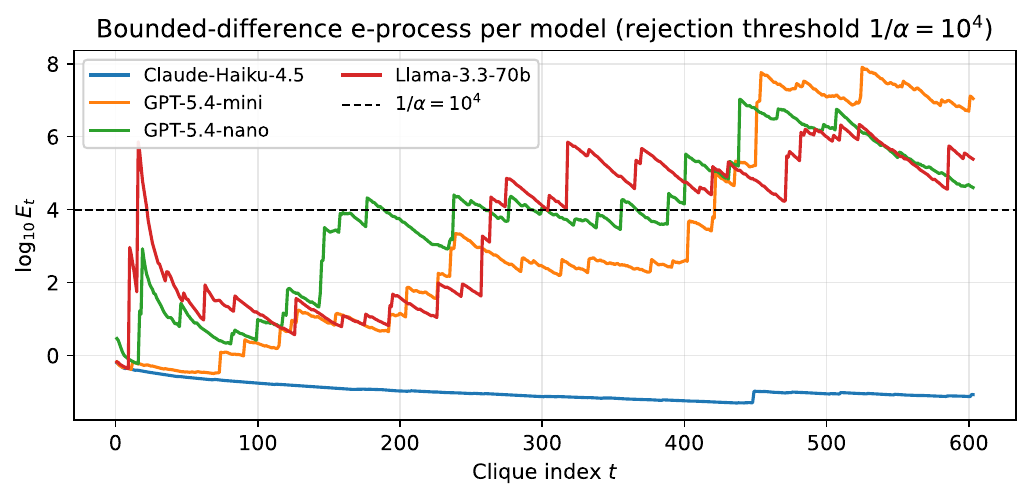}
\caption{\textbf{Bounded-difference e-process $E_t$ on
per-component residual streams (\Cref{thm:audit}).} Crossing
$1/\alpha=10^4$ at $t{=}16$ (Llama), $t{=}177$ (GPT-nano),
$t{=}422$ (GPT-5.4-mini); Claude-Haiku stays below~$1$.}
\label{fig:eprocess}
\end{figure}

\section{Clique-mining pipeline and reproducibility}
\label{app:cliques}

\paragraph{Paleka.} Tuples are pre-mined by the original authors:
each JSONL record is a single $(\mathrm{checker}, \{Q_i\},
\mathrm{relation})$ entry, parsed directly.

\paragraph{Polymarket.} We query gamma-api at the \emph{event}
level, retain events with $\ge\!2$ resolved markets, and emit one
of two relation types: (i)~a \emph{partition} relation
$\sum_i p_i = 1$ when an event has multi-candidate markets that
resolve with exactly one YES (e.g.\ election winners, sports
brackets); (ii)~an \emph{implication ladder} when an event has
multiple threshold markets ``X above $k$'' at ascending $k$.
Events are capped at $m{\le}8$ markets.

\paragraph{Compositional ensemble.}
For each clique with $m$ questions and $4$ LLMs, we draw $4$ random
seeds (independent across cliques and seeds, fixed master seed) and
for each seed assign each question $j$ to one of the $4$ LLMs
uniformly; the agent's quoted marginal on coordinate $j$ is the
$j$-th coordinate of the assigned LLM's JCD-projected forecast.
Hierarchical JCD applies $\Pijoint$ to the resulting concatenated
vector (in our experiments equal to direct JCD on the joint clique,
which by \Cref{thm:dykstra} is the converged Dykstra projection
through the local-then-coupling cycle).

\paragraph{Leakage control.} Every Polymarket event in the
evaluation slice resolved \emph{after} the latest model snapshot
(30 April 2026). Paleka cliques inherit
\citet{paleka2024consistency}'s own resolution-status verification.

\paragraph{Reproducibility manifest.}
\Cref{tab:repro} records the exact configuration used to produce
every number in the paper. The released artifact
(\url{https://github.com/akotawala10/composition-incoherence-icml})
contains the sample dumps, per-clique JCD residuals,
ensemble-assignment seeds, and the hierarchical-projection
pipeline.

\begin{table}[H]
\centering
\caption{Reproducibility manifest. $N$ values include
ensemble-assignment seeds where applicable.}
\label{tab:repro}
\small
\begin{tabular}{ll}
\toprule
\textbf{Item} & \textbf{Value} \\
\midrule
Mid-tier panel (\S\ref{sec:results}) & \texttt{claude-haiku-4-5-20251001}, \texttt{gpt-5.4-mini}, \texttt{gpt-5.4-nano}, \texttt{llama-3.3-70b-versatile} (Groq) \\
Frontier panel (\S\ref{subsec:frontier}) & \texttt{claude-opus-4-7}, \texttt{gpt-5.5} (Azure), \texttt{deepseek-v3.2}, \texttt{llama-4-maverick-17b-128e} (FP8) \\
Routing-ladder panel (App.~\ref{app:context}) & \texttt{claude-haiku-4-5-20251001}, \texttt{gpt-5.4-mini}, \texttt{gpt-5.4-nano}, \texttt{deepseek-v3.2}, \texttt{llama-4-maverick-17b-128e} ($5$ specialists; Llama-3.3-70b dropped on TPD-quota) \\
Snapshot date      & 30 April 2026 \\
$K$ samples / question & $8$ \\
Temperature        & $0.7$ \\
Sampling prompts   & released JSONL artifact \\
Paleka subset      & \texttt{scraped}, $134$ cliques per checker \\
Paleka post-filter & $134\!\times\!4$ checkers $=536$ cliques per model \\
Polymarket cliques & $268$ post-leakage-filter ($67\!\times\!4$ seeds) \\
Per-call Platt fit / eval split & $30\%$ / $70\%$ per (model, checker) \\
QP solver          & \textsc{OSQP} default tolerance ($1.4\!\times\!10^{-5}$) \\
Compositional master seed & \texttt{0} (numpy default RNG) \\
Random ensemble seeds & $4$ ($N{=}1{,}876$ bets) \\
Structured ensemble seeds & $1$ deterministic ($N{=}469$ bets) \\
\bottomrule
\end{tabular}
\end{table}
\vspace*{-0.3em}

\section{Polymarket cross-platform replication (per-call layer)}
\label{app:polymarket}
We replicate the headline per-call finding on Polymarket
($N{=}268$ held-out (model, clique) evaluations, summed across the
four-model panel). \Cref{fig:headline_polymarket} shows the same
incoherence-conditional gain shape as the Paleka result, with the
$m{\ge}4$ stratum (large multi-candidate partitions) contributing
the largest gains. \Cref{tab:main_polymarket} gives the per-model
breakdown.

\begin{figure}[H]
\centering
\includegraphics[width=0.65\linewidth]{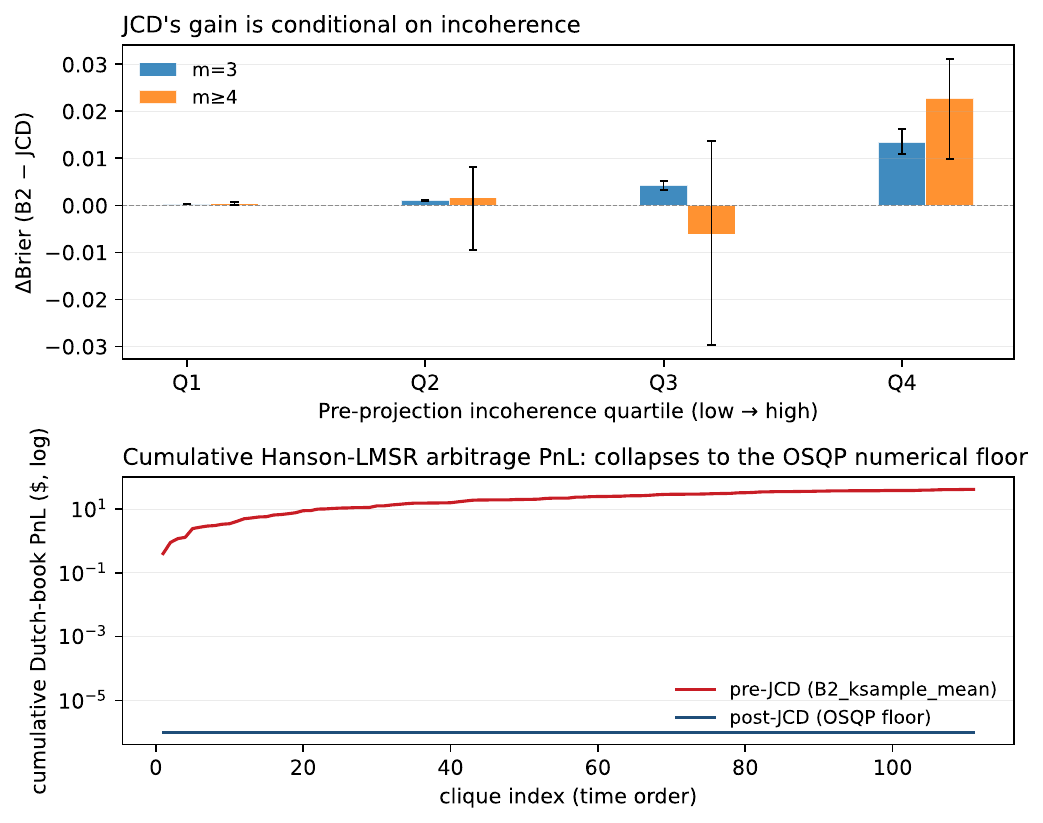}
\caption{Polymarket cross-platform replication of the per-call
result (\Cref{thm:pyth}): the incoherence-conditional gain
pattern of \Cref{thm:pyth} replicates on resolved Polymarket
events, with the largest mean gain in the highest-$\eps$
quartile at $m{\ge}4$ (note the smaller per-call absolute
magnitudes; the lower panel's cumulative Hanson--LMSR exposure
reduction to the QP floor matches Corollary~\ref{cor:noarb}).}
\label{fig:headline_polymarket}
\end{figure}

\begin{minipage}{\linewidth}
\centering
\refstepcounter{table}
\label{tab:main_polymarket}
{\small Table~\thetable. Real-LLM Brier ($\downarrow$) on Polymarket. Each row averages over that model's $\sim\!67$ held-out evaluation cliques (the $30\%$ slice not used for Platt fitting). The four rows together sum to $N{=}268$ (model, clique) evaluations across partition and threshold-ladder relations. $\Delta_{\mathrm{B2}\to\mathrm{JCD}}$ isolates the projection's effect; $\Delta_{\mathrm{B2}\to\mathrm{B6}}$ shows the Platt$+$JCD combination. Negative is better.\par}
\vspace{2pt}
\small
\begin{tabular}{l c c c c c c c}
\toprule
Model & B2 & B1 & JCD & B5 & B6 & $\Delta_{\mathrm{B2 \to JCD}}$ & $\Delta_{\mathrm{B2 \to B6}}$ \\
\midrule
anthropic\_claude-haiku-4-5-20251001 & 0.1551 & 0.1423 & 0.1409 & 0.1492 & 0.1435 & -0.0142 & -0.0115 \\
openai\_gpt-5.4-mini & 0.2368 & 0.1798 & 0.1796 & 0.1898 & 0.1803 & -0.0572 & -0.0565 \\
openai\_gpt-5.4-nano & 0.2387 & 0.1892 & 0.1892 & 0.1935 & 0.1839 & -0.0494 & -0.0547 \\
meta\_llama-3.3-70b & 0.2500 & 0.1872 & 0.1872 & 0.1970 & 0.1872 & -0.0628 & -0.0628 \\
\midrule
\textit{mean} & 0.2201 & 0.1746 & 0.1742 & 0.1824 & 0.1737 & -0.0459 & -0.0464 \\
\bottomrule
\end{tabular}
\end{minipage}

\section{Full per-cell tables}
\label{app:fulltables}
Per-model Brier under each baseline (\Cref{tab:main_real}), the
six-method ablation including B3/B4 (\Cref{tab:ablation_methods}),
the significance tests (\Cref{tab:sig_real}), and the Murphy
reliability/resolution decomposition (\Cref{tab:murphy_real}).

\begin{table*}[t]
\centering\small
\caption{Real-LLM Brier ($\downarrow$) on the \textbf{combined Paleka + Polymarket-partition} slice ($N{=}2412 = 603$ cliques $\times\,4$ models; $536$ Paleka cliques per model from four logical-relation checkers, plus $67$ Polymarket partition cliques per model). $\Delta_{\mathrm{B2}\to\mathrm{JCD}}$ isolates the projection's effect; $\Delta_{\mathrm{B2}\to\mathrm{B6}}$ shows the Platt$+$JCD combination. Negative is better.}
\label{tab:main_real}
\begin{tabular}{l c c c c c c c}
\toprule
Model & B2 & B1 & JCD & B5 & B6 & $\Delta_{\mathrm{B2 \to JCD}}$ & $\Delta_{\mathrm{B2 \to B6}}$ \\
\midrule
anthropic\_claude-haiku-4-5-20251001 & 0.1975 & 0.1931 & 0.1909 & 0.1643 & 0.1631 & -0.0066 & -0.0344 \\
openai\_gpt-5.4-mini & 0.2261 & 0.2112 & 0.2111 & 0.1886 & 0.1827 & -0.0150 & -0.0435 \\
openai\_gpt-5.4-nano & 0.2356 & 0.2212 & 0.2209 & 0.2015 & 0.1989 & -0.0148 & -0.0367 \\
meta\_llama-3.3-70b-versatile & 0.2485 & 0.2404 & 0.2406 & 0.2120 & 0.2100 & -0.0079 & -0.0385 \\
\midrule
\textit{mean} & 0.2269 & 0.2165 & 0.2159 & 0.1916 & 0.1887 & -0.0111 & -0.0383 \\
\bottomrule
\end{tabular}
\end{table*}

\begin{table}[t]\centering\footnotesize
\caption{Method-level ablation, averaged across the four-model panel and four Paleka logical-relation checkers. Each row is the across-(model, checker) mean Brier ($\downarrow$); $\Delta$ is reduction over $\mathrm{B2}$ raw $K$-sample marginals. Numbers are restricted to models that produced results under all sampling modes for fair comparison.}
\label{tab:ablation_methods}
\begin{tabular}{lrr}
\toprule
Method & Brier $\downarrow$ & $\Delta$ vs.\ B2 \\
\midrule
B2: Raw $K$-sample marginals & 0.2197 & 0.0000 (0.0\%) \\
B1: $K{=}1$ Karvetski projection & 0.2085 & -0.0112 (-5.1\%) \\
\textbf{JCD projection only} & 0.2076 & -0.0121 (-5.5\%) \\
\midrule
B3: Prompted self-consistency & 0.3144 & +0.0946 (+43.1\%) \\
B4: Direct joint elicitation & 0.1534 & -0.0663 (-30.2\%) \\
\midrule
B5: Platt scaling only & 0.1848 & -0.0350 (-15.9\%) \\
\textbf{B6: Platt $+$ JCD} & 0.1816 & -0.0382 (-17.4\%) \\
\bottomrule
\end{tabular}
\end{table}

\begin{table*}[t]\centering\small
\caption{Significance tests for $H_0$: mean Brier(B2) = mean Brier(JCD).}
\label{tab:sig_real}
\begin{tabular}{lcccc}
\toprule
Model & DM stat & DM $p$ & paired-boot stat & paired-boot $p$ \\
\midrule
anthropic\_claude-haiku-4-5-20251001 & 5.422 & 8.52e-08 & 0.0066 & 0 \\
openai\_gpt-5.4-mini & 10.682 & 1.657e-24 & 0.0150 & 0 \\
openai\_gpt-5.4-nano & 12.281 & 4.277e-31 & 0.0148 & 0 \\
meta\_llama-3.3-70b-versatile & 7.980 & 7.437e-15 & 0.0079 & 0 \\
\bottomrule
\end{tabular}
\end{table*}

\begin{table*}[t]\centering\small
\caption{Murphy decomposition of mean Brier across all real-LLM runs. JCD reduces reliability error without sacrificing resolution.}
\label{tab:murphy_real}
\begin{tabular}{lcccc}
\toprule
Baseline & REL $\downarrow$ & RES $\uparrow$ & UNC & Brier $\downarrow$ \\
\midrule
B2\_ksample\_mean & 0.0357 & 0.0173 & 0.2096 & 0.2280 \\
B1\_k1\_karvetski & 0.0233 & 0.0199 & 0.2096 & 0.2130 \\
JCD & 0.0229 & 0.0201 & 0.2096 & 0.2124 \\
B5\_platt & 0.0004 & 0.0250 & 0.2096 & 0.1850 \\
B6\_platt\_jcd & 0.0005 & 0.0279 & 0.2096 & 0.1822 \\
\bottomrule
\end{tabular}
\end{table*}

\section{Proof of Cor.~\ref{cor:magnitude} (Rayleigh-quotient magnitude bound)}
\label{app:magnitude}

We give the proof under uniform i.i.d.\ owner-selection
$\sigma:[m^\star]\!\to\![k]$, with $\bar\Pi := k^{-1}\sum_a \Pi_a$ and
$D := \mathrm{diag}(\Sigma_\Pi)$. Define $u := x - \bar\Pi$ where
$x_j := \Pi_{\sigma(j),j}$. Note $\bar\Pi \in \Mjoint$ by convexity
of $\Mjoint$, and pointwise $\E_\sigma[u_j] = 0$,
$\E_\sigma[u_j^2] = D_{jj}$, and cross-coordinate covariances vanish:
$\E_\sigma[u_j u_{j'}] = 0$ for $j\!\ne\!j'$ by independence of
$\sigma(j),\sigma(j')$.

\paragraph{Generic upper bound.}
By Prop.~\ref{prop:disagree} with reference $r=\bar\Pi$,
$\eps^\star \le \|u\|_2$, so
$\E_\sigma[(\eps^\star)^2] \le \E_\sigma[\|u\|_2^2]
= \sum_j D_{jj} = \mathrm{tr}(D)$.

\paragraph{Equality case ($\kappa_R = 1$).}
If $\Mjoint$ contains a single equality $a_R^\top r = b_R$ binding at
the projection (and no other constraint is active locally), the
projection has the closed form
$\Pijoint(x) = x - \frac{a_R^\top x - b_R}{\|a_R\|^2} a_R$, giving
$\eps^\star = |a_R^\top x - b_R|/\|a_R\|$. Since $\bar\Pi \in \Mjoint$
satisfies $a_R^\top \bar\Pi = b_R$, we have $a_R^\top x - b_R =
a_R^\top u$, so $\eps^{\star 2} = (a_R^\top u)^2/\|a_R\|^2$. Taking
expectation,
$\E_\sigma[(a_R^\top u)^2] = a_R^\top D a_R$ (cross-terms vanish), so
$\E_\sigma[(\eps^\star)^2] = a_R^\top D a_R / \|a_R\|^2$.

\paragraph{Inequality case with boundary $\bar\Pi$ ($\kappa_R = \tfrac12$).}
If the binding constraint at the projection is an inequality
$a_R^\top r \le b_R$ with $a_R^\top \bar\Pi = b_R$ (boundary), the
single-halfspace projection gives
$\eps^\star = (a_R^\top u)_+/\|a_R\|$. The variable $a_R^\top u$ is
zero-mean and symmetrically distributed about $0$ when the panel
$\{\Pi_a\}$ is symmetric in $a_R^\top \Pi_a$ around its mean $b_R$
(which holds asymptotically and approximately holds for our $k\!=\!4$
panel). Then $\E[(\cdot)_+^2] = \tfrac12 \E[(\cdot)^2]$, giving
$\E_\sigma[(\eps^\star)^2] = \tfrac12 \cdot a_R^\top D a_R/\|a_R\|^2$.

\paragraph{Inequality case with interior $\bar\Pi$.}
When $\bar\Pi$ is strictly interior to constraint $i$ (slack
$s_i := b_i - a_i^\top \bar\Pi > 0$), the violation
$v_i = (a_i^\top u - s_i)_+$ has reduced second moment under the
truncation, so $\E_\sigma[(\eps^\star)^2] < \tfrac12 a_R^\top D
a_R/\|a_R\|^2$ in this regime; the bound becomes a strict upper
bound. This matches the empirical conjunction under-shoot
($0.83\times$ predicted): $\bar\Pi_3 < \min(\bar\Pi_1, \bar\Pi_2)$ by
Jensen, leaving the upper Fr\'echet constraint slack at $\bar\Pi$.

\paragraph{Closed-form residuals.}
The negation case (equality, $a\!=\!(1,1)$, $b\!=\!1$, $\|a\|^2\!=\!2$)
recovers $\E[(\eps^\star)^2] = \mathrm{tr}(D)/2$. The partition case
(equality, $a\!=\!\mathbf 1_{m^\star}$, $b\!=\!1$, $\|a\|^2\!=\!m^\star$)
recovers $\mathrm{tr}(D)/m^\star$ in the no-clipping regime; with
clipping, the actual residual is strictly larger, leaving this as a
lower bound. The conjunction/disjunction cases use Fr\'echet
halfspace normals as $a_R$ (with $\|a_R\|^2 \in \{2,3\}$ depending on
the binding constraint per clique), with $\kappa_R = \tfrac12$.
$\hfill\square$

\paragraph{Empirical validation.}
\Cref{tab:magnitude_validation} reports observed
$\langle(\eps^\star)^2\rangle$ versus the Rayleigh-quotient prediction
$\kappa_R\, a_R^\top D a_R/\|a_R\|^2$ across $1{,}876$ ensemble
cliques. The dominant active constraint $a_R$ per clique is the
defining halfspace with smallest slack at $\bar\Pi$ (typically the
upper Fr\'echet for and, the upper Fr\'echet $r_3\!\le\!r_1\!+\!r_2$ for or).

\begin{center}
\refstepcounter{table}
\label{tab:magnitude_validation}
{\small Table~\thetable. Predicted vs.\ observed $\langle(\eps^\star)^2\rangle$ under uniform random owner-selection. ``Generic'' is $\mathrm{tr}(D)$ (the upper bound without single-active-constraint geometry); ``Rayleigh'' is $\kappa_R\,a_R^\top D a_R/\|a_R\|^2$.\par}
\small
\begin{tabular}{lrrrrrr}
\toprule
Relation & $\kappa_R$ & $\|a_R\|^2$ & $\langle$obs$\rangle$ & $\langle$Rayleigh$\rangle$ & ratio & per-clique $r$ \\
\midrule
\textsc{Neg}       & $1$           & $2$ & $0.0286$ & $0.0271$ & $1.054$ & $0.890$ \\
\textsc{Partition} & $1$           & $m^\star$ & $0.0146$ & $0.0137$ & $1.069$ & $0.850$ \\
\textsc{And}       & $\tfrac12$    & $\{2,3\}$ & $0.0153$ & $0.0184$ & $0.830$ & $0.710$ \\
\textsc{Or}        & $\tfrac12$    & $\{2,3\}$ & $0.0170$ & $0.0166$ & $1.026$ & $0.690$ \\
\bottomrule
\end{tabular}
\end{center}

The Rayleigh form is uniformly tight: ratios are $0.83$--$1.07$
across all four relations. The generic $\mathrm{tr}(D)$ bound is
loose by $1.9\times$ on \textsc{Neg}, $4.0\times$ on \textsc{Partition}
(at the average $m^\star\!=\!4$), and $5.3$--$6.1\times$ on
conjunction/disjunction; the Rayleigh form recovers the typical
magnitude from the specialist panel covariance alone.

\section{Compositional ablation: projection-step decomposition}
\label{app:projection_ablation}

We decompose hierarchical JCD into two layers (per-component JCD
then joint projection) and four operators evaluated against
resolved labels:
(A) raw composed (no per-component JCD, no joint projection),
(B) JCD composed (per-component JCD only),
(C) raw $+$ joint projection,
(D) JCD $+$ joint projection (hierarchical).
\Cref{tab:proj_ablation} reports mean $\eps^\star$ before projection
(A, B) and mean compositional Brier under each operator.

\begin{center}
\refstepcounter{table}
\label{tab:proj_ablation}
{\small Table~\thetable. Compositional projection-step ablation. $\eps^\star$ is post-aggregation, restricted to the resolved-label subset (used for the paired Brier computation in the rightmost columns); the §5.1 headline mean of $0.118$ on partition uses the full $1{,}876$-clique ensemble, while operator B here reports $0.098$ on the smaller resolved-bet slice. Hier $\equiv$ D.\par}
\small
\begin{tabular}{llrrrr}
\toprule
Relation & Op & $\langle\eps^\star\rangle$ & exposure & Brier & $\Delta$Brier vs.\ A \\
\midrule
\textsc{Neg} & A: raw            & $0.144$ & $0.203$ & $0.2367$ & \\
             & B: JCD            & $0.114$ & $0.161$ & $0.2301$ & $-0.0066$ \\
             & C: raw $+\Pi^\star$ & $0$     & $0$     & $0.2174$ & $-0.0193$ \\
             & D: hier            & $0$     & $0$     & $0.2165$ & $-0.0202$ \\
\midrule
\textsc{And} & A & $0.076$ & $0.110$ & $0.2270$ & \\
             & B & $0.054$ & $0.079$ & $0.2186$ & $-0.0084$ \\
             & C & $0$     & $0$     & $0.2144$ & $-0.0126$ \\
             & D & $0$     & $0$     & $0.2110$ & $-0.0160$ \\
\midrule
\textsc{Or}  & A & $0.110$ & $0.153$ & $0.2538$ & \\
             & B & $0.072$ & $0.102$ & $0.2557$ & $+0.0019$ \\
             & C & $0$     & $0$     & $0.2571$ & $+0.0033$ \\
             & D & $0$     & $0$     & $0.2584$ & $+0.0046$ \\
\midrule
\textsc{Partition} & A & $0.312$ & $0.620$ & $0.2165$ & \\
                   & B & $0.098$ & $0.193$ & $0.1864$ & $-0.0301$ \\
                   & C & $0$     & $0$     & $0.1832$ & $-0.0333$ \\
                   & D & $0$     & $0$     & $0.1817$ & $-0.0348$ \\
\bottomrule
\end{tabular}
\end{center}

Per-component JCD dominates the partition gain (the local repair already brings
the assembled mass close to $1$). The joint projection dominates negation, and,
and disjunction, where local repair leaves cross-coordinate violations untouched.
The disjunction inversion under (B)--(D) is the artefact discussed in
\S\ref{sec:exp}.

\section{Greedy-decoding ($T\!=\!0$) same-model control}
\label{app:t0_control}

To isolate within-model decoding stochasticity from
coordinate-isolation disagreement, we re-ran the $4$-fresh-seed
Claude-Haiku-4.5 same-model protocol of \S\ref{sec:exp} at
temperature $T\!=\!0$ on $60$ cliques per Paleka relation and $30$
Polymarket partitions. Numbers in \Cref{tab:t0} are mean $\eps^\star$
under $16$ uniform random owner-selection draws per clique; the
$T\!=\!0.7$ row is the sampling-temperature baseline of \S\ref{sec:exp}.

\begin{center}
\refstepcounter{table}
\label{tab:t0}
{\small Table~\thetable. Same-model $\eps^\star$ at $T\!=\!0.7$ vs.\ $T\!=\!0$ (greedy decoding). Cross-model column is the random-assignment ensemble restricted to this control's clique slice ($60$ Paleka cliques per relation, $30$ Polymarket partitions), not the full-ensemble headline of \S\ref{subsec:compres_results}.\par}
\small
\begin{tabular}{lrrrrr}
\toprule
Relation & cross-model & same-model $T\!=\!0.7$ & same-model $T\!=\!0$ & $T_0/T_{0.7}$ & $T_0/$cross \\
\midrule
\textsc{Neg}       & $0.114$ & $0.025$ & $0.096$ & $3.83$ & $0.84$ \\
\textsc{And}       & $0.054$ & $0.019$ & $0.023$ & $1.20$ & $0.42$ \\
\textsc{Or}        & $0.072$ & $0.040$ & $0.074$ & $1.87$ & $1.04$ \\
\textsc{Partition} & $0.098$ & $0.058$ & $0.025$ & $0.44$ & $0.26$ \\
\bottomrule
\end{tabular}
\end{center}

The greedy-decoding control rules out the alternative ``$\eps^\star$
is just sampling noise'' hypothesis. On negation and disjunction,
$T\!=\!0$ \emph{increases} the residual ($T_0/T_{0.7}\!>\!1$); on
partition it decreases it by $56\%$, indicating that some partition
residual is decoding-stochastic. On all four relations the residual
remains positive at $T\!=\!0$ and is comparable in magnitude to
(or, for negation/disjunction, larger than) the $T\!=\!0.7$
same-model figure. The result implicates genuine
coordinate-isolation disagreement: the same model, prompted with
isolated single-question contexts at greedy decoding, produces
different ``deterministic'' outputs across coordinates that fail to
satisfy the cross-component coupling. (Anthropic does not accept a
\texttt{seed} parameter, so $4$ ``specialists'' at $T\!=\!0$ vary only
through residual non-determinism in batch processing; $\eps^\star$
remains detectable nonetheless.)

\section{$K$-sweep on real LLM forecasts}
\label{app:ablations}

\paragraph{Per-component layer.}
\Cref{fig:ksweep} sweeps $K \in \{4, 8, 16, 24\}$ on Anthropic
Claude-Haiku-4.5 NegChecker forecasts. The non-zero
$\Delta\mathrm{Brier}$ floor ($0.008$--$0.010$) is the empirical
signature of $\pF \notin \MC$ rather than finite-sample fluctuation.

\begin{figure}[!htbp]
\centering
\includegraphics[width=0.38\linewidth]{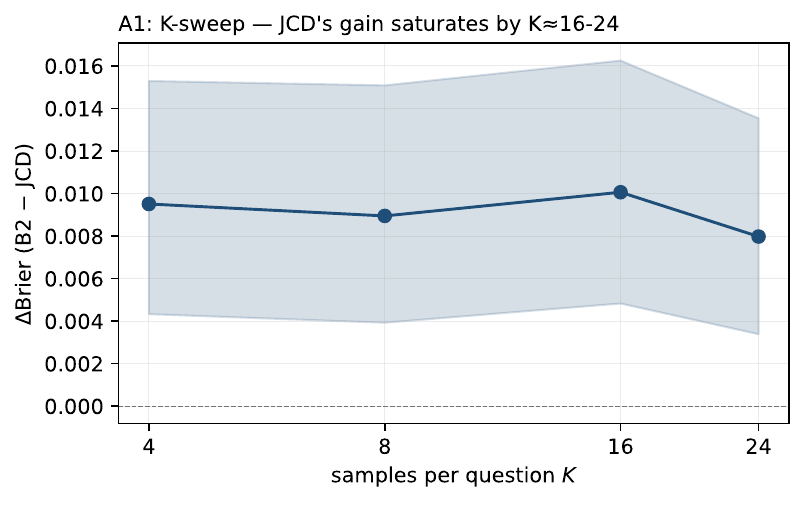}
\caption{$K$-sweep on Anthropic Claude-Haiku-4.5 NegChecker
forecasts. Non-vanishing $\Delta\mathrm{Brier}$ asymptote: the
empirical signature of structural incoherence in $\pF$.}
\label{fig:ksweep}
\end{figure}

\paragraph{Compositional layer.}
We also sweep $K \in \{8, 16, 32\}$ at the compositional layer (the
residual $\eps^\star$ in \S\ref{sec:exp}). $K\!=\!8$ uses the
existing $4$-LLM panel and $K\!=\!16, 32$ extends Anthropic
Claude-Haiku-4.5, GPT-5.4-mini, and GPT-5.4-nano to additional
samples on a $40$-clique-per-relation, $20$-partition subset (Groq
Llama-$3.3$-70b excluded after exhausting its tokens-per-day
quota mid-sweep). The 3-model panel ensemble is computed identically
to the 4-LLM panel under uniform random owner-selection, $16$
ensemble draws per clique.

\begin{center}
\refstepcounter{table}
\label{tab:k_sweep_comp}
{\small Table~\thetable. Compositional $\eps^\star$ vs.\ $K$, 3-model panel.\par}
\small
\begin{tabular}{lrrrr}
\toprule
Relation & $K\!=\!8$ & $K\!=\!16$ & $K\!=\!32$ & $N$ \\
\midrule
\textsc{Neg}       & $0.143$ & $0.145$ & $0.143$ & $432$ \\
\textsc{And}       & $0.067$ & $0.067$ & $0.067$ & $432$ \\
\textsc{Or}        & $0.078$ & $0.080$ & $0.079$ & $432$ \\
\textsc{Partition} & $0.076$ & $0.082$ & $0.082$ & $224$ \\
\bottomrule
\end{tabular}
\end{center}

The compositional residual is flat across $K\!\in\!\{8, 16, 32\}$ (within
$\pm 0.003$ per relation), confirming that $\eps^\star$ is a
\emph{structural} property of cross-component coupling under
specialist routing rather than finite-sample fluctuation. Combined
with the per-component flatness above, the K-sweep robustness covers
both layers.

\section{Structured-routing replication}
\label{app:structured}

\Cref{fig:structured} repeats the compositional experiment under
the structured-routing topology of \Cref{sec:exp}.

The residual is positive on $98.5\%$ of negations (vs.\ $66\%$ under
random), $36\%$ of partitions (vs.\ $94\%$), $12\%$ of
conjunctions (vs.\ $33\%$), and $16\%$ of disjunctions
(vs.\ $43\%$). Per-bet exposure is $\$0.11/$bet ($\$51.88$ over
$469$ bets), comparable to the random regime's $\$0.13$/bet.
Negation has a higher residual rate because two specialists disagree more strongly
than four random ones; conjunction/disjunction improve because
structured routing reduces cross-model mixing, while partitions
improve because one specialist owns all candidates.

\begin{figure}[!htbp]
\centering
\includegraphics[width=0.58\linewidth]{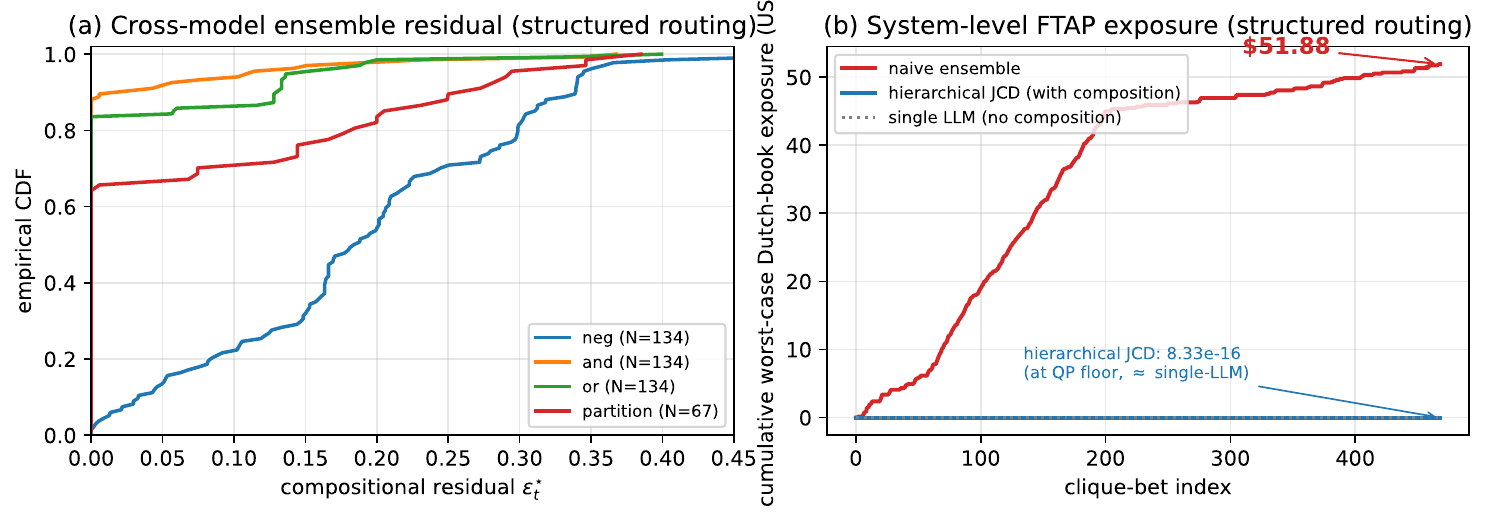}
\caption{\textbf{Structured-routing replication of
\Cref{fig:compositional}.} The same certificate is positive with a
different relation breakdown but comparable per-bet exposure
($\$0.11$/bet vs.\ $\$0.13$/bet under random assignment). The lower
cumulative endpoint reflects fewer bets under structured routing,
$N{=}469$ vs.\ $N{=}1{,}876$, not lower per-bet exposure.}
\label{fig:structured}
\end{figure}

\section{Worst-case failure modes}
\label{app:failure_modes}

\Cref{tab:failure_modes} lists the top-$5$ ensemble cliques by
compositional residual $\eps^\star$ under random assignment. The
largest cases are realized resolved-Paleka quotes where specialists
make opposite sides of a negation jointly too likely, or make a
disjunction inconsistent with its antecedents.

\begin{minipage}{\linewidth}
\centering
\refstepcounter{table}
\label{tab:failure_modes}
{\small Table~\thetable. Top-5 ensemble cliques by compositional
residual $\eps^\star$ under random assignment ($N{=}1{,}876$).\par}
\tiny
\begin{tabular}{rlrlll}
\toprule
\# & relation & $m$ & component quotes (assigned LLM) & violated constraint & $\varepsilon^\star$ \\
\midrule
1 & neg & 2 & 0.84 (Claude), 0.89 (Llama) & $p_1+p_2=1.731\ne 1$ & 0.517 \\
2 & or & 3 & 0.00 (Claude), 0.02 (GPT), 0.92 (Claude) & $p_3=0.916<\max(p_1,p_2)=0.020$ & 0.517 \\
3 & or & 3 & 0.02 (Claude), 0.03 (GPT), 0.92 (Claude) & $p_3=0.924<\max(p_1,p_2)=0.031$ & 0.502 \\
4 & or & 3 & 0.02 (Claude), 0.74 (GPT), 0.03 (Claude) & $p_3=0.033<\max(p_1,p_2)=0.740$ & 0.500 \\
5 & neg & 2 & 0.71 (GPT), 0.99 (Claude) & $p_1+p_2=1.696\ne 1$ & 0.492 \\
\bottomrule
\end{tabular}
\end{minipage}

\section{Planner-style routing simulation: per-question detail}
\label{app:realagent}

We report all $100$ live-style partition forecasting questions
used in the planner-style routing simulation (\Cref{sec:exp}).
Each row gives the partition label, the number of outcomes $m$,
the per-outcome marginal returned by the randomly-assigned
specialist LLM (LLM key: \texttt{C}=Claude-Haiku-4.5,
\texttt{M}=GPT-5.4-mini, \texttt{N}=GPT-5.4-nano,
\texttt{L}=Llama-3.3-70b), the resulting partition mass
$\sum_i p_i$, and the compositional residual $\eps^\star$. All
four specialist clients ran with $K{=}8$ verbalized-probability
samples at temperature $0.7$; each received its assigned outcome
as an isolated single-question prompt with no context about the
partition or the other outcomes. Master seed $0$, sampling
without replacement when $m\!\le\!4$. Total cost:
$\approx\!2{,}400$ LLM calls, $\approx\!\$5$,
$\approx\!35$ minutes wall-clock. Rows are sorted by descending
$\eps^\star$ so the worst cases appear first.

\begin{longtable}{p{0.36\linewidth} c l r r}
\caption{Planner-style routing simulation, all $N{=}100$ partitions, sorted by $\eps^\star$. Each specialist sees a single Bernoulli question; the agent assembles into a partition quote. LLM key: \texttt{C}=Claude-Haiku-4.5, \texttt{M}=GPT-5.4-mini, \texttt{N}=GPT-5.4-nano, \texttt{L}=Llama-3.3-70b. All 100 partitions exhibit positive compositional residual.}\label{tab:realagent}\\
\toprule
Partition & $m$ & per-outcome quote (LLM) & $\sum p_i$ & $\eps^\star$ \\
\midrule
\endfirsthead
\multicolumn{5}{l}{\small \emph{(continued from previous page)}}\\
\toprule
Partition & $m$ & per-outcome quote (LLM) & $\sum p_i$ & $\eps^\star$ \\
\midrule
\endhead
\midrule
\multicolumn{5}{r}{\small \emph{(continued on next page)}}\\
\endfoot
\bottomrule
\endlastfoot
Largest US AI startup IPO 2026 & 4 & \scriptsize 0.39(C), 0.73(N), 0.67(L), 0.71(M) & 2.497 & 0.749 \\
Top MMLU score, year-end 2026 & 4 & \scriptsize 0.25(C), 0.62(M), 0.57(N), 0.85(L) & 2.290 & 0.650 \\
2026 Nobel Chemistry laureate field & 4 & \scriptsize 0.59(C), 0.38(N), 0.57(L), 0.61(M) & 2.164 & 0.582 \\
Top SWE-Bench Verified score, year-end 2026 & 4 & \scriptsize 0.35(N), 0.61(L), 0.25(C), 0.86(M) & 2.079 & 0.540 \\
Next UK general election year & 4 & \scriptsize 0.32(N), 0.23(L), 0.72(C), 0.74(M) & 2.008 & 0.504 \\
2026 ACM Turing Award area & 3 & \scriptsize 0.58(M), 0.54(C), 0.74(L) & 1.870 & 0.502 \\
Top HumanEval+ score, year-end 2026 & 4 & \scriptsize 0.25(C), 0.48(N), 0.43(L), 0.78(M) & 1.936 & 0.468 \\
Top GPQA score, year-end 2026 & 4 & \scriptsize 0.34(N), 0.36(C), 0.42(L), 0.80(M) & 1.925 & 0.462 \\
2026 EU AI Act state of enforcement & 3 & \scriptsize 0.30(N), 0.74(M), 0.73(L) & 1.775 & 0.447 \\
2026 German Bundestag composition (post-election) & 3 & \scriptsize 0.72(C), 0.42(L), 0.59(N) & 1.735 & 0.424 \\
2026 Nobel Economics laureate field & 4 & \scriptsize 0.50(L), 0.44(N), 0.35(M), 0.53(C) & 1.813 & 0.406 \\
2026 Emmy Outstanding Drama Series platform & 3 & \scriptsize 0.74(L), 0.67(M), 0.29(N) & 1.701 & 0.405 \\
Top global cloud provider revenue share 2026 & 4 & \scriptsize 0.72(N), 0.39(L), 0.08(C), 0.58(M) & 1.768 & 0.403 \\
Next iPhone announcement quarter & 3 & \scriptsize 0.70(M), 0.80(L), 0.17(C) & 1.676 & 0.395 \\
2026 Ohio gubernatorial party & 3 & \scriptsize 0.71(C), 0.40(N), 0.55(L) & 1.660 & 0.381 \\
2026 US House majority & 3 & \scriptsize 0.44(N), 0.58(L), 0.64(M) & 1.660 & 0.381 \\
Top open-source LLM family year-end 2026 & 4 & \scriptsize 0.71(C), 0.23(M), 0.39(N), 0.43(L) & 1.759 & 0.379 \\
2026 International Booker Prize winner language & 3 & \scriptsize 0.68(C), 0.43(M), 0.54(N) & 1.651 & 0.376 \\
2027 UK Prime Minister tenure & 3 & \scriptsize 0.52(C), 0.41(M), 0.71(L) & 1.649 & 0.375 \\
2026 Wolf Prize in Mathematics & 3 & \scriptsize 0.77(M), 0.35(C), 0.50(L) & 1.618 & 0.357 \\
UN climate ambition state, COP31 (2026) & 3 & \scriptsize 0.59(N), 0.71(C), 0.29(M) & 1.591 & 0.341 \\
Apple WWDC 2026 keynote & 3 & \scriptsize 0.54(N), 0.79(C), 0.25(M) & 1.579 & 0.334 \\
Largest US tech IPO of 2026 by company type & 4 & \scriptsize 0.43(L), 0.36(C), 0.44(M), 0.43(N) & 1.665 & 0.333 \\
ITER first plasma timing & 3 & \scriptsize 0.34(N), 0.50(L), 0.73(C) & 1.573 & 0.331 \\
2026 Nobel Physics laureate field & 4 & \scriptsize 0.43(C), 0.44(L), 0.30(M), 0.47(N) & 1.636 & 0.318 \\
2027 Hugo Award Best Novel author origin & 3 & \scriptsize 0.70(L), 0.17(N), 0.65(M) & 1.529 & 0.305 \\
2027 ICC T20 Cricket World Cup winner region & 3 & \scriptsize 0.55(C), 0.42(N), 0.55(L) & 1.526 & 0.304 \\
2026 Cannes Palme d'Or winner nationality & 3 & \scriptsize 0.57(C), 0.27(M), 0.68(L) & 1.518 & 0.299 \\
2026 H1 BTC return bucket & 3 & \scriptsize 0.42(L), 0.38(N), 0.71(C) & 1.515 & 0.297 \\
2026 Nobel Medicine laureate field & 4 & \scriptsize 0.20(N), 0.36(M), 0.28(C), 0.75(L) & 1.589 & 0.294 \\
2026 Grammy Album of the Year genre & 4 & \scriptsize 0.52(M), 0.28(C), 0.36(L), 0.42(N) & 1.577 & 0.289 \\
2026 FIFA World Cup winner confederation & 3 & \scriptsize 0.56(C), 0.42(N), 0.52(L) & 1.498 & 0.287 \\
Top semiconductor maker by revenue, 2026 & 3 & \scriptsize 0.57(M), 0.57(C), 0.35(N) & 1.494 & 0.285 \\
US unemployment rate, June 2026 & 4 & \scriptsize 0.71(C), 0.17(N), 0.28(L), 0.38(M) & 1.540 & 0.270 \\
ECB main refinancing rate, December 2026 & 3 & \scriptsize 0.49(N), 0.55(M), 0.42(L) & 1.460 & 0.266 \\
2027 Six Nations rugby winner & 3 & \scriptsize 0.39(M), 0.64(C), 0.42(L) & 1.450 & 0.260 \\
2027 Mexico midterm legislative balance & 3 & \scriptsize 0.72(C), 0.28(N), 0.44(L) & 1.442 & 0.255 \\
2026 US CPI year-over-year, December print & 4 & \scriptsize 0.30(N), 0.46(M), 0.42(L), 0.33(C) & 1.507 & 0.254 \\
2026 US real GDP year-over-year growth & 4 & \scriptsize 0.34(N), 0.56(L), 0.26(M), 0.34(C) & 1.499 & 0.249 \\
2027 French Presidential winner ideology & 3 & \scriptsize 0.33(C), 0.52(M), 0.58(L) & 1.430 & 0.248 \\
Next FOMC decision (June 2026 meeting) & 3 & \scriptsize 0.33(N), 0.68(C), 0.42(L) & 1.424 & 0.245 \\
2026 Nobel Literature laureate region & 3 & \scriptsize 0.62(C), 0.35(N), 0.43(L) & 1.399 & 0.230 \\
2026 IPhO top country & 3 & \scriptsize 0.24(N), 0.42(L), 0.73(C) & 1.394 & 0.227 \\
Top semiconductor revenue 2026: foundry vs IDM & 3 & \scriptsize 0.25(C), 0.50(L), 0.63(N) & 1.380 & 0.219 \\
2026 Tour de France winner nationality & 3 & \scriptsize 0.26(L), 0.40(N), 0.72(C) & 1.377 & 0.218 \\
December 2026 FOMC decision & 3 & \scriptsize 0.49(M), 0.46(C), 0.42(L) & 1.371 & 0.214 \\
2027 Australian Open men's singles winner & 4 & \scriptsize 0.10(C), 0.25(M), 0.35(L), 0.72(N) & 1.419 & 0.209 \\
2027 Breakthrough Prize Life Sciences focus & 4 & \scriptsize 0.27(N), 0.35(C), 0.50(L), 0.30(M) & 1.417 & 0.209 \\
2026 Academy Award Best Picture winner type & 3 & \scriptsize 0.72(C), 0.43(L), 0.20(N) & 1.353 & 0.204 \\
2027 BAFTA Best Film origin & 3 & \scriptsize 0.36(C), 0.71(L), 0.29(M) & 1.351 & 0.203 \\
2026 New York gubernatorial party & 3 & \scriptsize 0.83(M), 0.26(N), 0.24(C) & 1.336 & 0.194 \\
2026 NCAA Division I football national champion conference & 3 & \scriptsize 0.44(N), 0.28(C), 0.60(L) & 1.321 & 0.185 \\
2026 Brazilian general election composition & 3 & \scriptsize 0.57(N), 0.32(C), 0.42(L) & 1.309 & 0.178 \\
June 2027 FOMC decision & 3 & \scriptsize 0.42(L), 0.66(C), 0.21(M) & 1.295 & 0.170 \\
WHO PHEIC declarations 2026 & 3 & \scriptsize 0.16(N), 0.44(C), 0.11(M) & 0.706 & 0.170 \\
2026 F1 Constructors' Championship winner & 4 & \scriptsize 0.27(N), 0.38(L), 0.18(M), 0.50(C) & 1.336 & 0.168 \\
UK CPI YoY November 2026 & 3 & \scriptsize 0.54(L), 0.31(N), 0.43(M) & 1.276 & 0.159 \\
2026 Mercury Music Prize winner gender & 3 & \scriptsize 0.47(C), 0.40(L), 0.40(M) & 1.268 & 0.154 \\
2026 H1 SP500 return bucket & 4 & \scriptsize 0.35(N), 0.50(L), 0.21(C), 0.25(M) & 1.302 & 0.151 \\
2026 Nobel Peace Prize laureate type & 3 & \scriptsize 0.75(N), 0.44(M), 0.06(C) & 1.252 & 0.148 \\
2026 Tony Best Musical genre & 3 & \scriptsize 0.35(C), 0.27(M), 0.63(N) & 1.254 & 0.147 \\
World population, mid-2027 & 3 & \scriptsize 0.22(M), 0.42(C), 0.60(N) & 1.245 & 0.141 \\
2026 F1 Drivers' Championship winner & 4 & \scriptsize 0.29(M), 0.26(N), 0.35(L), 0.36(C) & 1.259 & 0.129 \\
EUR/USD year-end 2026 & 3 & \scriptsize 0.21(M), 0.69(C), 0.29(N) & 1.204 & 0.118 \\
Gold price year-end 2026 & 3 & \scriptsize 0.35(C), 0.12(N), 0.74(M) & 1.204 & 0.118 \\
WTI crude oil year-end 2026 price bucket & 3 & \scriptsize 0.45(M), 0.43(L), 0.31(N) & 1.199 & 0.115 \\
2026 Pennsylvania US Senate election & 3 & \scriptsize 0.45(C), 0.48(L), 0.24(N) & 1.175 & 0.101 \\
Tesla 2026 Q4 deliveries bucket & 3 & \scriptsize 0.35(C), 0.50(L), 0.32(M) & 1.170 & 0.098 \\
2026 Atlantic hurricane season activity & 3 & \scriptsize 0.39(N), 0.47(M), 0.30(C) & 1.165 & 0.095 \\
2027 Berlin Golden Bear winner type & 3 & \scriptsize 0.69(N), 0.09(M), 0.07(L) & 0.847 & 0.088 \\
2027 H1 SP500 return bucket & 3 & \scriptsize 0.09(N), 0.34(M), 0.43(L) & 0.848 & 0.088 \\
2026 Wimbledon men's singles winner & 4 & \scriptsize 0.12(C), 0.08(N), 0.12(M), 0.85(L) & 1.172 & 0.086 \\
2026 Florida gubernatorial party & 3 & \scriptsize 0.58(L), 0.25(C), 0.31(N) & 1.140 & 0.081 \\
US 10-year Treasury yield year-end 2026 & 3 & \scriptsize 0.30(M), 0.42(C), 0.42(L) & 1.139 & 0.080 \\
2026 Olivier Best New Play origin & 3 & \scriptsize 0.70(C), 0.38(M), 0.05(L) & 1.136 & 0.079 \\
2026 IMO top country & 3 & \scriptsize 0.36(N), 0.27(M), 0.50(L) & 1.134 & 0.077 \\
2026 mpox global status & 3 & \scriptsize 0.49(M), 0.13(N), 0.25(C) & 0.867 & 0.076 \\
2026 H2 SP500 return bucket & 4 & \scriptsize 0.12(N), 0.29(C), 0.28(L), 0.17(M) & 0.859 & 0.071 \\
2026 global temperature ranking & 3 & \scriptsize 0.27(N), 0.11(M), 0.50(L) & 0.883 & 0.068 \\
2026 PGA Tour FedEx Cup winner nationality & 2 & \scriptsize 0.68(M), 0.23(N) & 0.905 & 0.067 \\
Cancer screening trial readout 2026 & 3 & \scriptsize 0.23(M), 0.66(C), 0.23(L) & 1.114 & 0.066 \\
2026 Texas gubernatorial party & 3 & \scriptsize 0.55(N), 0.16(C), 0.18(M) & 0.895 & 0.061 \\
2026 US Senate majority & 3 & \scriptsize 0.58(L), 0.27(M), 0.26(N) & 1.104 & 0.060 \\
Largest 2026 US data-center investment site & 3 & \scriptsize 0.37(M), 0.23(N), 0.50(L) & 1.098 & 0.056 \\
2026 Arctic sea ice minimum extent & 3 & \scriptsize 0.62(C), 0.13(N), 0.34(M) & 1.095 & 0.055 \\
2026 Japan LDP leadership & 3 & \scriptsize 0.67(L), 0.33(N), 0.09(C) & 1.090 & 0.052 \\
2026 ACM Gödel Prize area & 3 & \scriptsize 0.43(C), 0.15(M), 0.50(L) & 1.085 & 0.049 \\
2026 NBA Finals winner conference & 2 & \scriptsize 0.53(L), 0.53(M) & 1.062 & 0.044 \\
March 2027 FOMC decision & 3 & \scriptsize 0.35(C), 0.29(N), 0.42(L) & 1.059 & 0.034 \\
2026 NHL Stanley Cup winner conference & 2 & \scriptsize 0.50(C), 0.54(M) & 1.035 & 0.025 \\
2026 World Series winner league & 2 & \scriptsize 0.53(M), 0.51(N) & 1.034 & 0.024 \\
2026 California gubernatorial party & 3 & \scriptsize 0.50(N), 0.16(C), 0.38(L) & 1.038 & 0.022 \\
2026 net global solar capacity addition & 3 & \scriptsize 0.20(C), 0.44(M), 0.33(N) & 0.965 & 0.020 \\
Top gold-medal nation, 2028 Summer Olympics & 3 & \scriptsize 0.58(L), 0.19(M), 0.26(N) & 1.031 & 0.018 \\
2028 US Presidential winner & 3 & \scriptsize 0.46(L), 0.53(M), 0.02(C) & 1.016 & 0.009 \\
Year-end 2026 leading mobile OS & 3 & \scriptsize 0.85(L), 0.15(M), 0.02(C) & 1.016 & 0.009 \\
GLP-1 agonist label expansion 2026 & 3 & \scriptsize 0.30(N), 0.57(M), 0.11(C) & 0.984 & 0.009 \\
FDA Alzheimer's drug action 2026 & 3 & \scriptsize 0.33(M), 0.18(N), 0.50(L) & 1.012 & 0.007 \\
2026 Booker Prize winner nationality & 3 & \scriptsize 0.37(M), 0.18(N), 0.44(L) & 0.996 & 0.002 \\
2026 Pulitzer Prize Fiction winner setting & 3 & \scriptsize 0.72(C), 0.25(M), 0.03(N) & 0.999 & 0.001 \\
\end{longtable}

\paragraph{Worst-case anatomy.}
The largest residual ($\eps^\star=0.749$) is the ``Largest US AI
startup IPO 2026'' partition (four sectoral outcomes:
infrastructure, model, applications, other). Three of the four
assigned specialists privately quote $P(\text{their sector}) > 0.5$
(the fourth is at $0.39$); collectively the specialists
over-allocate, with assigned masses summing to $2.50$. Each
marginal is locally well-calibrated for the single Bernoulli the
specialist was asked, but no specialist saw that the four sectors
tile the field; as a result, the assembled system quotes $250\%$
probability mass. Simplex projection reduces the quote to a
proper distribution in closed form, eliminating the exposure to
QP-tolerance. ML-benchmark partitions (Top MMLU,
$\eps^\star{=}0.65$; SWE-Bench Verified, $0.54$; HumanEval+,
$0.47$; GPQA, $0.46$) form a coherent high-residual cluster: in
each, four specialists are assigned to one candidate model
organization, and the specialists collectively over-allocate;
masses sum to $\sim\!2.0$ even when only one or two specialists
are above $0.5$ individually.

\paragraph{Why some partitions yield near-zero residual.}
The five smallest residuals (Pulitzer Prize Fiction setting
$0.001$, Booker Prize nationality $0.002$, FDA Alzheimer's drug
action $0.007$, GLP-1 label expansion $0.009$, leading mobile OS
$0.009$) share a common structure: the partition either has only
two real candidate outcomes, or has a strongly bimodal
training-set prior. The single-question marginals then
approximately satisfy the partition constraint without any
explicit cross-component coordination. The certificate is
therefore discriminating rather than uniformly positive: it is
positive on novel multi-way partitions and near zero on
well-rehearsed ones.

\section{Routing-protocol ladder: per-partition detail}
\label{app:context}

We compare seven routing protocols on the same $100$ live-style
multi-candidate partition benchmark (App.~\ref{app:realagent}). Three
protocols use fixed owner-selected routing under different prompt
templates (\textsc{isolated}, \textsc{listed}, \textsc{full}); four
use a tool-using planner (Claude-Haiku-4.5 or GPT-5.5 as planner,
each $\times$ unguided/coherence-guided framings). Specialist panel
across all protocols: Claude-Haiku-4.5, GPT-5.4-mini,
GPT-5.4-nano, DeepSeek-V3.2, Llama-4-Maverick-17B-128E
($5$ specialists; Llama-3.3-70b dropped after TPD-quota issues).
$K\!=\!8$ verbalized samples per specialist at temperature $0.7$.
Total cost: $\approx\!\$80$ across all protocols.

\paragraph{Prompt templates.}
\textsc{isolated} reuses the per-question
\texttt{DEFAULT\_PROMPT}
of \citet{paleka2024consistency}-style routing.
\textsc{listed} prepends a numbered list of all $m$ partition outcomes
followed by ``Specific outcome to estimate: $\{\textit{outcome}\}$'';
no normalization instruction.
\textsc{full} adds the explicit clause: ``these outcomes are mutually
exclusive and collectively exhaustive, so the probabilities you would
assign to all $m$ outcomes must sum to $1$. Bear this in mind.''

\begin{center}
\refstepcounter{table}
\label{tab:context_full}
{\small Table~\thetable. Routing-protocol ladder ($N\!=\!100$
partitions). Cost is mean delegate calls per partition divided by
mean partition size $\bar m \!\approx\! 3.3$. CI bracketed.\par}
\small
\begin{tabular}{lcccc}
\toprule
protocol & $\langle\eps^\star\rangle$ & median & frac.\ $\eps^\star\!>\!10^{-3}$ & cost \\
\midrule
\multicolumn{5}{l}{\emph{Owner-selected routing, prompt-engineering only:}}\\
\textsc{isolated}             & $0.235$ {\footnotesize $[0.20,0.27]$} & $0.186$ & $1.00$ & $1\!\times\!m$ \\
\textsc{listed}               & $0.090$ {\footnotesize $[0.08,0.11]$} & $0.077$ & $0.98$ & $1\!\times\!m$ \\
\textsc{full+sum-to-1}        & $0.081$ {\footnotesize $[0.07,0.09]$} & $0.070$ & $0.99$ & $1\!\times\!m$ \\
\midrule
\multicolumn{5}{l}{\emph{Tool-using planner with delegate $+$ submit-final tools:}}\\
Claude-Haiku, unguided        & $0.0090$ {\footnotesize $[0.001,0.024]$} & $0.000$ & $0.05$ & $4.69\!\times\!m$ \\
Claude-Haiku, guided          & $0.0044$ {\footnotesize $[0.000,0.011]$} & $0.000$ & $0.03$ & $4.87\!\times\!m$ \\
GPT-5.5,    unguided        & $0.0000$ {\footnotesize $[0,\,0]$}        & $0.000$ & $0.00$ & $6.98\!\times\!m$ \\
GPT-5.5,    guided          & $0.0000$ {\footnotesize $[0,\,0]$}        & $0.000$ & $0.00$ & $6.51\!\times\!m$ \\
\midrule
\multicolumn{5}{l}{\emph{Geometric repair (any prompt, owner-selected):}}\\
hierarchical JCD              & $\le 1.5\!\times\!10^{-16}$ (QP floor) & $0$ & $0$ & $1\!\times\!m + \Pi^\star$ \\
\bottomrule
\end{tabular}\\[2pt]
{\small Pairwise Wilcoxon: prompt \textsc{iso}\,$\to$\,\textsc{listed}
$p\!=\!2.2\!\times\!10^{-12}$, but \textsc{listed}\,$\to$\,\textsc{full}
$p\!=\!0.18$ (showing alternatives matters; sum-to-$1$ instruction
does not). Planner Claude-Haiku $>$ GPT-5.5:
$p\!<\!10^{-3}$ (unguided), $p\!<\!10^{-4}$ (guided): reasoning depth
strictly closes the residual that tool-use alone leaves.\par}
\end{center}

\paragraph{Where context helps most.}
The largest drops (\textsc{iso}\,$\to$\,\textsc{full}) are on
partitions where isolated specialists strongly anchor to their
assigned outcome with no cross-context knowledge of alternatives:
``Largest US AI startup IPO 2026'' ($0.74\!\to\!0.12$), ``Top MMLU
score, year-end 2026'' ($0.60\!\to\!0.05$), ``2026 ACM Turing Award area''
($0.63\!\to\!0.13$). These are partitions over emerging or
specialist-knowledge candidates where the specialist needs to know
the alternatives to allocate probability mass appropriately.

\paragraph{Where context hurts.}
Eighteen partitions show $\eps^\star_{\text{full}}\!>\!\eps^\star_{\text{iso}}$,
predominantly political/economic ones where isolated specialists were
already roughly normalised (e.g., 2028 US Presidential winner
$0.003\!\to\!0.022$, 2026 California gubernatorial party
$0.10\!\to\!0.26$). When isolated estimates already average to $\sim 1$,
forcing specialists to reason about all alternatives can introduce
mass that wasn't there.

\paragraph{Implication for deployment.}
Three regimes emerge: (i)~\emph{prompt-engineering at fixed
owner-selection} attenuates the failure mode but bottoms out at
$\eps^\star\!\approx\!0.08$ (essentially the
\textsc{isolated}\,$\to$\,\textsc{listed} step); (ii)~\emph{tool-using
planners} reach near-zero residual but pay $5$--$7\!\times$ the naive
elicitation budget, with frontier reasoning depth required to close
the last $\sim\!0.005$ that tool-use alone leaves; (iii)~\emph{hierarchical
projection} reaches the same near-zero residual at the naive
$1\!\times\!m$ elicitation cost, modulo a small QP solve. The
geometric repair is therefore the cost-efficient route to coherence
when the coupling set $\Cset$ is specifiable; smart-planner
integration is the alternative when $\Cset$ is implicit and the
planner can infer it through extra delegate calls.

\section{Tool-augmented specialist A/B: per-partition detail}
\label{app:toolaug}

We re-run the first $30$ partitions of the case study under a
tool-augmented protocol (\S\ref{sec:exp}). Each specialist
issues exactly one DuckDuckGo \texttt{web\_search} call against
its assigned outcome text, and the top-five title+snippet results
are inserted as context before the specialist emits its $K{=}8$
verbalized-probability samples. Master seed and assignment are
identical to the no-tool first-$30$ baseline, so the per-partition
$\Delta\eps^\star = \eps^\star_t - \eps^\star_b$ is a direct A/B.
Total cost: $\approx\!720$ LLM calls + $90$ search calls,
$\approx\!\$2$, $\approx\!25$ minutes wall-clock.

Aggregate: mean $\eps^\star_b = 0.260 \to \eps^\star_t = 0.283$;
$20/30$ partitions are equal or worse with retrieval, $10/30$ are
better; mean $\Delta\eps^\star = +0.023$. The single largest
\emph{worsening} is ``World population, mid-2027'' ($+0.563$):
every specialist's search returns recent UN bracket estimates
around $8.1$--$8.2$\,bn and each anchors near
$P(\text{their bucket}){\approx}0.7$, summing to $2.22$. The
single largest \emph{improvement} is ``2026 Grammy Album of the
Year genre'' ($-0.267$): retrieval surfaces 2026 Grammy nominee
distributions that pull every specialist toward a similar
partition shape.

\begin{figure}[H]
\centering
\includegraphics[width=0.65\linewidth]{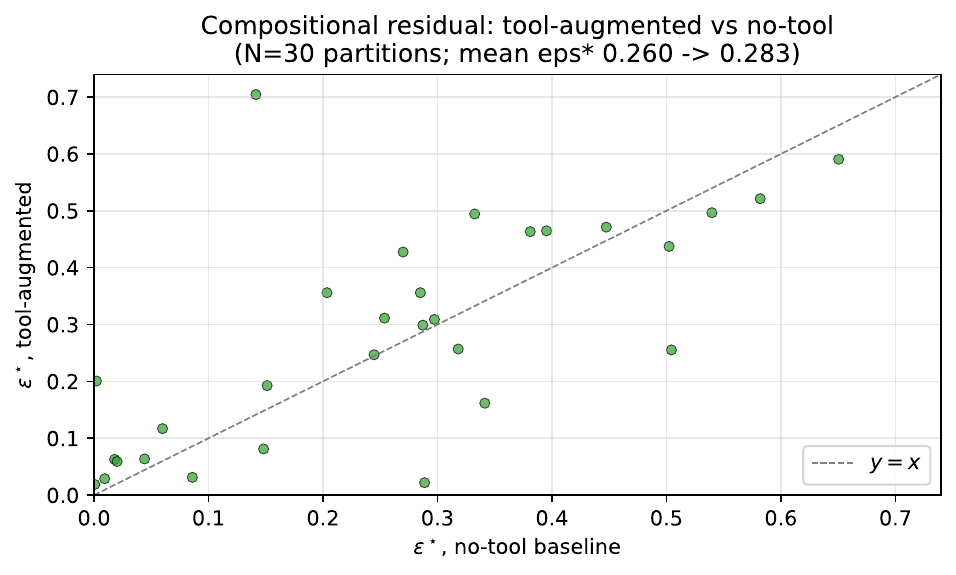}
\caption{\textbf{Tool augmentation does not eliminate the
compositional residual.} Each point is one of the $30$ matched
partitions; $x$-axis is $\eps^\star$ in the no-tool baseline,
$y$-axis is $\eps^\star$ when each specialist receives top-five
DuckDuckGo snippets for its assigned outcome before quoting.
Most points lie on or above $y\!=\!x$: tool augmentation tightens
within-specialist variance but cross-specialist disagreement
persists, so the assembled partition quote is no more coherent
with retrieval than without.}
\label{fig:toolaug}
\end{figure}

\begin{longtable}{p{0.46\linewidth} r r r r}
\caption{Tool-augmented specialist A/B, all $30$ matched partitions, sorted by $|\Delta\eps^\star|$. $\sum_t$ is the partition mass under tool-augmented retrieval; $\eps^\star_b$ and $\eps^\star_t$ are the no-tool baseline and tool-augmented compositional residuals.}\label{tab:toolaug}\\
\toprule
Partition & $\sum_t$ & $\eps^\star_b$ & $\eps^\star_t$ & $\Delta\eps^\star$ \\
\midrule
\endfirsthead
\multicolumn{5}{l}{\small \emph{(continued)}}\\
\toprule
Partition & $\sum_t$ & $\eps^\star_b$ & $\eps^\star_t$ & $\Delta\eps^\star$ \\
\midrule
\endhead
\midrule
\multicolumn{5}{r}{\small \emph{(continued on next page)}}\\
\endfoot
\bottomrule
\endlastfoot
World population, mid-2027 & 2.220 & 0.141 & 0.704 & +0.563~$\uparrow$ \\
2026 Grammy Album of the Year genre & 1.044 & 0.289 & 0.022 & -0.267~$\downarrow$ \\
Next UK general election year & 0.489 & 0.504 & 0.255 & -0.249~$\downarrow$ \\
2026 Booker Prize winner nationality & 1.348 & 0.002 & 0.201 & +0.198~$\uparrow$ \\
UN climate ambition state, COP31 (2026) & 1.280 & 0.341 & 0.162 & -0.180~$\downarrow$ \\
Largest US tech IPO of 2026 by company type & 1.989 & 0.333 & 0.494 & +0.162~$\uparrow$ \\
US unemployment rate, June 2026 & 1.855 & 0.270 & 0.427 & +0.158~$\uparrow$ \\
2026 Academy Award Best Picture winner type & 1.616 & 0.204 & 0.356 & +0.152~$\uparrow$ \\
2026 US House majority & 1.802 & 0.381 & 0.463 & +0.082~$\uparrow$ \\
Top semiconductor maker by revenue, 2026 & 1.566 & 0.285 & 0.356 & +0.071~$\uparrow$ \\
Next iPhone announcement quarter & 1.805 & 0.395 & 0.465 & +0.069~$\uparrow$ \\
2026 Nobel Peace Prize laureate type & 1.131 & 0.148 & 0.081 & -0.067~$\downarrow$ \\
2026 ACM Turing Award area & 1.757 & 0.502 & 0.437 & -0.065~$\downarrow$ \\
2026 Nobel Physics laureate field & 1.514 & 0.318 & 0.257 & -0.061~$\downarrow$ \\
2026 Nobel Chemistry laureate field & 2.042 & 0.582 & 0.521 & -0.061~$\downarrow$ \\
Top MMLU score, year-end 2026 & 2.181 & 0.650 & 0.591 & -0.060~$\downarrow$ \\
2026 US CPI year-over-year, December print & 1.623 & 0.254 & 0.311 & +0.058~$\uparrow$ \\
2026 US Senate majority & 1.203 & 0.060 & 0.117 & +0.057~$\uparrow$ \\
2026 Wimbledon men's singles winner & 0.938 & 0.086 & 0.031 & -0.055~$\downarrow$ \\
Top gold-medal nation, 2028 Summer Olympics & 1.109 & 0.018 & 0.063 & +0.045~$\uparrow$ \\
Top SWE-Bench Verified score, year-end 2026 & 1.970 & 0.540 & 0.497 & -0.043~$\downarrow$ \\
2026 H1 SP500 return bucket & 1.385 & 0.151 & 0.193 & +0.041~$\uparrow$ \\
2026 net global solar capacity addition & 0.897 & 0.020 & 0.059 & +0.039~$\uparrow$ \\
2026 EU AI Act state of enforcement & 1.816 & 0.447 & 0.471 & +0.024~$\uparrow$ \\
2028 US Presidential winner & 1.050 & 0.009 & 0.029 & +0.019~$\uparrow$ \\
2026 NBA Finals winner conference & 1.090 & 0.044 & 0.064 & +0.019~$\uparrow$ \\
2026 Pulitzer Prize Fiction winner setting & 1.032 & 0.001 & 0.019 & +0.018~$\uparrow$ \\
2026 H1 BTC return bucket & 1.535 & 0.297 & 0.309 & +0.012~$\uparrow$ \\
2026 FIFA World Cup winner confederation & 1.518 & 0.287 & 0.299 & +0.012~$\uparrow$ \\
Next FOMC decision (June 2026 meeting) & 1.427 & 0.245 & 0.247 & +0.002~$\uparrow$ \\
\end{longtable}

\section{Per-specialist residual attribution}
\label{app:attribution}

For each ensemble bet with $\eps^\star\!>\!10^{-9}$
($n_{\text{pos}}\!=\!894$ across the 1,876-bet ensemble of
\S\ref{sec:exp}), we attribute residual mass to each specialist
by computing the $L_2$ norm of the residual vector
$(x{-}\Pi^\star(x))$ restricted to the coordinates that specialist
owned. The decomposition gives a per-clique credit assignment for
which sub-model drives the joint incoherence.

\begin{table}[h]
\centering
\small
\begin{tabular}{@{}lccc@{}}
\toprule
Specialist            & cliques owned & mean $\|\delta_a\|_2$ & median \\
\midrule
Claude-Haiku-4.5      & $506$         & $0.092$               & $0.078$ \\
GPT-5.4-mini          & $605$         & $0.120$               & $0.114$ \\
GPT-5.4-nano          & $491$         & $0.093$               & $0.080$ \\
Llama-3.3-70b         & $464$         & $0.098$               & $0.090$ \\
\bottomrule
\end{tabular}

\caption{Per-specialist mean attributed $L_2$ mass on coordinates
each specialist owned, across the $894$ positive-residual bets.}
\label{tab:attribution}
\end{table}

GPT-5.4-mini carries the largest attributed mass per clique among
the four mid-tier specialists, consistent with the frontier-panel
finding (\S\ref{subsec:frontier}) that the GPT-5.5 upgrade
reduces magnitude but not prevalence: even with a frontier roster,
$\eps^\star\!>\!0$ on $97.8\%$ of the matched partition bets.

\section{Trivial vs.\ hierarchical projection: per-relation gaps}
\label{app:projgap}

We verify that the closed-form local projections of
App.~\ref{app:size3} agree with the iterative Boyle--Dykstra cycle
of \Cref{thm:dykstra} to numerical precision on the equality-couple
relations (negation, partition) and require the cyclic iteration
to recover the exact $L_2$ projection on the inequality-couple
relations (conjunction, disjunction), where the Fr\'echet polytope
is not a Cartesian product of half-spaces.

\begin{table}[h]
\centering
\small
\begin{tabular}{@{}lccc@{}}
\toprule
Relation     & $n$ bets & max gap                  & mean gap                 \\
\midrule
\textsc{Neg} & $536$    & $1.1\!\times\!10^{-16}$  & $1.7\!\times\!10^{-17}$  \\
\textsc{Partition} & $268$ & $0$ (exact)           & $0$                      \\
\textsc{And} & $536$    & $7.8\!\times\!10^{-2}$   & $1.3\!\times\!10^{-3}$   \\
\textsc{Or}  & $536$    & $8.8\!\times\!10^{-2}$   & $3.2\!\times\!10^{-3}$   \\
\bottomrule
\end{tabular}

\caption{Per-relation maximum
$\|\Pi^{\text{closed}}{-}\Pi^{\text{JCD}}\|_\infty$ over the
1,876 ensemble bets. Negation and partition match to numerical
precision; conjunction/disjunction require the cyclic projection
because the Fr\'echet polytope is not a Cartesian product of
half-spaces.}
\label{tab:projgap}
\end{table}

The conjunction/disjunction mean gaps ($\sim\!10^{-3}$) are
well below the per-clique $\eps^\star$ magnitudes
($0.058$--$0.118$ on these relations,
\S\ref{subsec:compres_results}), so the empirical residuals
reported throughout are not artefacts of any approximate
projection step; the Dykstra cycle converges to the exact
$L_2$ projection on every clique reported.

\section{Causal-mechanism (coupling-visibility) experiment}
\label{app:coupling}

This appendix supports the coupling-visibility intervention of
\S\ref{subsec:coupling}.

\paragraph{Sample of partitions.}
The 20 partitions are the full planner-harness slate of
\S\ref{subsec:planner_disc}; the spec ``20 highest-$\eps^\star$''
selects every partition because the harness has
$\eps^\star\!>\!0$ on $20/20$.

\paragraph{Routing.}
Held fixed across \textsc{Blind} and \textsc{Informed} at the
planner-harness's chosen specialist-to-outcome assignment per
partition. The original planner harness has no random-assignment
seeds; the four ``seeds'' here are independent $K\!=\!8$ sampling
rounds per (partition, condition), giving four paired $\eps^\star$
values per partition.

\paragraph{Prompts.}
\textsc{Blind} is a single-Bernoulli question on the assigned
outcome with no partition context.
\textsc{Informed} additionally shows the partition label, all
sibling outcomes, the explicit $\sum_i p_i\!=\!1$ constraint,
and the peers' \textsc{Blind} quotes on the other outcomes.
\textsc{Informed} uses \textsc{Blind} quotes from the same seed
as context, so the two conditions are tightly paired within a
(partition, seed) cell.

\paragraph{Aggregate result.}
Mean paired
$\Delta\eps^\star = \eps^\star_{\textsc{Blind}}{-}\eps^\star_{\textsc{Informed}} = +0.221$
($95\%$ paired-bootstrap CI $[+0.173, +0.270]$; Wilcoxon
$W{=}257$, $p\!=\!2.6\!\times\!10^{-10}$; $n_{\text{pairs}}\!=\!77/80$
cells surviving parsing). $16/20$ partitions improve under
\textsc{Informed}, $1/20$ unchanged, $3/20$ worsened. The three
worsening partitions ($\Delta\!<\!0$) are non-payoff multi-way
forecasts where \textsc{Blind} specialists were already roughly
normalised; \textsc{Informed} can introduce mass that the
isolated quotes did not have. This matches the LLM-side
mitigations result (\S\ref{subsec:mitigations}): prompt-based
fixes can harm already-coherent quotes, and the geometric
certificate gates them.

\end{document}